\documentclass[english]{article}
\usepackage[T1]{fontenc}
\usepackage[latin9]{inputenc}
\usepackage{geometry}
\usepackage{babel}
\usepackage{mathtools}
\usepackage{dsfont}
\usepackage{amsmath}
\usepackage{amssymb}
\usepackage[bookmarks=false,
 breaklinks=false,pdfborder={0 0 1},colorlinks=false]
 {hyperref}
\hypersetup{
 colorlinks,linkcolor=red,anchorcolor=blue,citecolor=blue}

\makeatletter
\usepackage{babel}
\usepackage{babel}
\usepackage{babel}

\usepackage{xcolor,colortbl}
\definecolor{Gray}{gray}{0.85}
\usepackage{tcolorbox}

\usepackage{mathtools}

\usepackage{cite}\usepackage{amsthm}\usepackage{dsfont}\usepackage{array}\usepackage{mathrsfs}\usepackage{comment}\onecolumn
\usepackage{natbib}
\usepackage{color}\usepackage{babel}

\newcommand{\mymid}{\,|\,} 

\allowdisplaybreaks

\usepackage{enumitem}
\setlist[itemize]{leftmargin=1.5em}
\setlist[enumerate]{leftmargin=1.5em}

\usepackage{babel}
\usepackage{algorithm}
\usepackage{algorithmic}
\usepackage{arydshln}

\numberwithin{equation}{section}

\definecolor{yxc}{RGB}{255,0,0}
\definecolor{yjc}{RGB}{125,0,0}
\definecolor{cm}{RGB}{0,0,200}
\definecolor{yly}{RGB}{0,150,0}

\makeatother

\begin{document}
\theoremstyle{plain} \newtheorem{lemma}{\textbf{Lemma}} \newtheorem{prop}{\textbf{Proposition}}\newtheorem{theorem}{\textbf{Theorem}}\setcounter{theorem}{0}
\newtheorem{corollary}{\textbf{Corollary}} \newtheorem{assumption}{\textbf{Assumption}}
\newtheorem{example}{\textbf{Example}} \newtheorem{definition}{\textbf{Definition}}
\newtheorem{fact}{\textbf{Fact}} \newtheorem{condition}{\textbf{Condition}}\theoremstyle{definition}

\theoremstyle{remark}\newtheorem{remark}{\textbf{Remark}}\newtheorem{claim}{\textbf{Claim}}\newtheorem{conjecture}{\textbf{Conjecture}}
\title{Adaptivity and Convergence of Probability Flow ODEs in Diffusion Generative
Models}
\author{Jiaqi Tang\thanks{Department of Statistics, University of Wisconsin-Madison, Madison,
WI 07302, USA; Email: \texttt{tang274@wisc.edu}.}\and Yuling Yan\thanks{Department of Statistics, University of Wisconsin-Madison, Madison,
WI 07302, USA; Email: \texttt{yuling.yan@wisc.edu}.}}

\maketitle
\begin{abstract}
Score-based generative models, which transform noise into data by
learning to reverse a diffusion process, have become a cornerstone
of modern generative AI. This paper contributes to establishing theoretical
guarantees for the probability flow ODE, a widely used diffusion-based
sampler known for its practical efficiency. While a number of prior
works address its general convergence theory, it remains unclear whether
the probability flow ODE sampler can adapt to the low-dimensional
structures commonly present in natural image data. We demonstrate
that, with accurate score function estimation, the probability flow
ODE sampler achieves a convergence rate of $O(k/T)$ in total variation
distance (ignoring logarithmic factors), where $k$ is the intrinsic
dimension of the target distribution and $T$ is the number of iterations.
This dimension-free convergence rate improves upon existing results
that scale with the typically much larger ambient dimension, highlighting
the ability of the probability flow ODE sampler to exploit intrinsic
low-dimensional structures in the target distribution for faster sampling.
\end{abstract}


\setcounter{tocdepth}{2}\tableofcontents{}

\section{Introduction\protect\label{sec:intro}}

Diffusion models have emerged as a groundbreaking advancement in the
realm of generative models, achieving state-of-the-art performance
in producing high-quality, diverse data across a wide range of applications,
including image synthesis \citep{rombach2022high,ramesh2022hierarchical,saharia2022photorealistic},
audio generation \citep{kong2021diffwave}, video generation \citep{villegas2022phenaki},
and molecular modeling \citep{hoogeboom2022equivariant}. These models,
often known as score-based generative models (SGMs), operate by simulating
a diffusion process that progressively corrupts data into noise, followed
by learning a reverse process to reconstruct samples from the target
distribution. Their success lies in their ability to effectively model
complex, high-dimensional data distributions while offering a scalable
and flexible framework for training and sampling. For comprehensive
overviews of recent developments in this field, see e.g.,  \citet{yang2022diffusion,croitoru2023diffusion}.

SGMs are built upon two stochastic processes in $\mathbb{R}^{d}$,
called a forward process and a reverse process. The forward process
progressively corrupts data drawn from the target distribution $p_{\mathsf{data}}$
into noise, transforming it as
\[
X_{0}\rightarrow X_{1}\rightarrow\cdots\rightarrow X_{T},
\]
where $X_{0}\sim p_{\mathsf{data}}$ and $X_{T}$ approximates Gaussian
noise. The reverse process then reconstructs the data by iteratively
refining Gaussian noise $Y_{T}$ into a sample $Y_{0}$, such that
the distribution of $Y_{0}$ closely resembles $p_{\mathsf{data}}$: 

\[
Y_{T}\rightarrow Y_{T-1}\rightarrow\cdots\rightarrow Y_{0},
\]
To ensure accurate sampling from the target distribution, the reverse
process is designed to align the marginal distributions of $Y_{t}$
and $X_{t}$ at each step. Inspired by classical results on time-reversal
of stochastic differential equations (SDEs) \citep{anderson1982reverse,haussmann1986time},
SGMs construct the reverse process using the gradients of the log
marginal density of the forward process, known as (Stein) score functions.
Specifically, each $Y_{t-1}$ is generated from $Y_{t}$ with the
aid of the score function $\nabla\log p_{X_{t}}(\cdot)$, where $p_{X_{t}}$
denotes the density of $X_{t}$. Two mainstream SGMs, Denoising Diffusion
Probabilistic Model (DDPM) \citep{ho2020denoising} and Denoising Diffusion
Implicit Models (DDIM) \citep{song2020denoising}, operate under this
framework. Since the exact score functions are typically unavailable,
they are approximated by neural networks trained via score-matching
techniques \citep{hyvarinen2005estimation,hyvarinen2007some,vincent2011connection,song2019generative}. 

In recent years, substantial progress has been made in establishing
general convergence results for both DDPM- and DDIM-type samplers
\citep{benton2023linear,chen2023restoration,benton2023error,huang2024convergence,li2024unified,li2023towards,chen2022sampling,chen2023probability,chen2022improved,tang2024score,li2024sharp,li2024d,gao2024convergence}. This line of research treats the score matching as a black box procedure, and mainly focus on understanding how number of iterations $T$ and score matching error affect the convergence rate. These general theories aim to accommodate the broadest possible class
of target distributions without imposing conditions like smoothness or log-concavity. Assuming accurate score function estimation,
the best-known convergence rates for these samplers under the total variation (TV) distance are $O(d/T)$,
where $d$ is the ambient data dimension and $T$ is the number of
steps \citep{li2024d,li2024sharp}. However, such general results often
fail to reflect the practical performance of these samplers, particularly
in scenarios where the target data distribution exhibits low-dimensional
structure. 

Empirical evidence suggests that natural image distributions are often
concentrated on or near low-dimensional manifolds embedded in higher-dimensional spaces \citep{pope2021intrinsic,simoncelli2001natural}. This has motivated
a line of recent research investigating the adaptivity of DDPM-type samplers
to these intrinsic structures \citep{azangulov2024convergence,wang2024diffusion,tang2024adaptivity,tang2024conditional,chen2023score,li2024adapting,li2024d,huang2024denoising,potaptchik2024linear,oko2023diffusion}.
A notable finding related to this paper is that, with carefully chosen
coefficients \citep{li2024adapting}, the convergence rate of the DDPM sampler scales with the intrinsic dimension $k$ instead of the ambient dimension $d$, demonstrating its adaptivity to unknown low-dimensional structures.

Despite these advances, it remains unclear whether deterministic counterparts,
such as the probability flow ODE sampler \citep{song2020score} (a case of DDIM-type sampler \citep{song2020denoising}), exhibit similar adaptivity to unknown low-dimensional structures. This question is particularly important as deterministic samplers often converge much faster in practice than their stochastic counterparts. Understanding the theoretical
underpinnings of deterministic samplers in the presence of low-dimensional structures is thus a critical research challenge, and is the main focus of this paper. 

\paragraph{Our contributions.}

In this work, we address this question by providing new theoretical
guarantees for the probability flow ODE sampler.
Specifically, we show that with appropriate coefficient design, the
TV distance between the target and generated distributions
is upper bounded by $O(k/T)$ up to some logarithmic factors and additive
score matching errors, where $k$ is the intrinsic dimension of the
support of $p_{\mathsf{data}}$ defined using the notion of metric entropy. This result bridges the gap between
theory and practice, demonstrating that the probability flow ODE sampler
can automatically adapt to intrinsic low-dimensional structures for
faster convergence. Additionally, our analysis introduces a streamlined
framework that is more concise than existing analysis for the probability
flow ODE \citep{li2024sharp}.

Finally, while finalizing this work, we became aware of an independent
study \citet{liang2025low}, posted on arXiv on January 22, 2025,
also tackles this problem. A detailed comparison between our results
and theirs is provided in Section~\ref{sec:main-results}. 

\section{Set-up}

In this section, we review key components of SGMs and the probability
flow ODE sampler, providing a formal description of the model.

\paragraph{Forward process.}

We consider a Markov process starting from $X_{0}\sim p_{\mathsf{data}}$,
the target data distribution, and proceeding as:
\begin{equation}
X_{t}=\sqrt{1-\beta_{t}}X_{t-1}+\sqrt{\beta_{t}}W_{t}\quad(t=1,\ldots,T),\label{eq:forward-update}
\end{equation}
where $W_{1},\ldots,W_{T}$ are i.i.d.~samples from $\mathcal{N}(0,I_{d})$,
and $\beta_{t}\in(0,1)$ are step-dependent learning rates. Following
prior theoretical work on diffusion models \citep{li2023towards,li2024adapting,li2024d},
we adopt the following learning rate schedule:
\begin{equation}
\beta_{1}=\frac{1}{T^{c_{0}}},\quad\beta_{t+1}=\frac{c_{1}\log T}{T}\min\Big\{\beta_{1}\Big(1+\frac{c_{1}\log T}{T}\Big)^{t},1\Big\}\quad(t=1,\ldots,T-1),\label{eq:learning-rate}
\end{equation}
where $c_{0},c_{1}>0$ are sufficiently large constants. For convenience,
we define $\alpha_{t}\coloneqq1-\beta_{t}$ and $\overline{\alpha}_{t}\coloneqq\prod_{i=1}^{t}\alpha_{i}$
for $1\leq t\leq T$. Using these, $X_{t}$ can be expressed as:
\begin{equation}
X_{t}=\sqrt{\overline{\alpha}_{t}}X_{0}+\sqrt{1-\overline{\alpha}_{t}}\,\overline{W}_{t},\label{eq:forward-formula}
\end{equation}
where $\overline{W}_{t}\sim\mathcal{N}(0,I_{d})$ is independent of
$X_{0}$. The learning rate schedule (\ref{eq:learning-rate}) ensures
that $\beta_{t}=O(\log T/T)$ is small across the forward process,
and $\overline{\alpha}_{T}$ becomes vanishingly small, making $X_{T}$
close to $\mathcal{N}(0,I_{d})$ in distribution. Any schedule satisfying
the conditions in Lemma~\ref{lem:T1} can be used to achieve the
same results presented in this paper.

\paragraph{Score functions and score estimates.}

The reverse process reconstructs samples from $p_{\mathsf{data}}$
by inverting the forward diffusion. This relies on the (Stein) score
functions, which are the gradients of the log-density of the marginal
distributions of the forward process:
\[
s_{t}^{\star}(x)\coloneqq\nabla\log p_{X_{t}}(x)\quad(t=1,\ldots,T),
\]
where $p_{X_{t}}(\cdot)$ is the density function of $X_{t}$. Since
the exact score functions are typically unknown, they are approximated
by learned score estimates $s_{t}:\mathbb{R}^{d}\to\mathbb{R}^{d}$.
To quantify the error in these approximations, we define: \begin{subequations}\label{subeq:score-error}
\begin{align}
\varepsilon_{\mathsf{score},t}(x)^{2} & \coloneqq\big\| s_{t}(x)-s_{t}^{\star}(x)\big\|_{2}^{2}+\big[(s_{t}(x)-s_{t}^{\star}(x))^{\top}s_{t}^{\star}(x)\big]^{2},\quad\text{and}\label{eq:score-error-entrywise}\\
\varepsilon_{\mathsf{Jacobi},t}(x)^{2} & \coloneqq\mathsf{Tr}\big(J_{s_{t}}(x)-J_{s_{t}^{\star}}(x)\big)^{2}+\big\| J_{s_{t}}(x)-J_{s_{t}^{\star}}(x)\big\|_{\mathrm{F}}^{2},\label{eq:Jacob-error-entrywise}
\end{align}
\end{subequations}where $J_{s_{t}^{\star}}$ and $J_{s_{t}}$ denote
the Jacobian matrices of $s_{t}^{\star}$ and $s_{t}$, respectively.
As it turns out, our theoretical guarantees depends on the average
score and Jacobian errors as follows:
\begin{equation}
\varepsilon_{\mathsf{score}}^{2}\coloneqq\frac{1}{T}\sum_{t=1}^{T}\mathbb{E}\big[\varepsilon_{\mathsf{score},t}(X_{t})^{2}\big],\quad\text{and}\quad\varepsilon_{\mathsf{Jacobi}}^{2}\coloneqq\frac{1}{T}\sum_{t=1}^{T}\mathbb{E}\big[\varepsilon_{\mathsf{Jacobi},t}(X_{t})^{2}\big].\label{eq:score-error}
\end{equation}
These two terms quantifies the average $\ell_{2}$ score estimation
error between the true and learned score functions $s_{t}^{\star}(x)$
and $s_{t}(x)$, and their Jacobian matrices $J_{s_{t}^{\star}}(x)$
and $J_{s_{t}}(x)$ respectively. 

\paragraph{Probability flow ODE.}

The probability flow ODE sampler \citep{song2020score} is a special case of the DDIM-type samplers
\citep{song2020denoising} that constructs the reverse process
deterministically. Starting from $Y_{T}\sim\mathcal{N}(0,I_{d})$,
it evolves as:
\begin{equation}
Y_{t-1}=\frac{1}{\sqrt{\alpha_{t}}}\big(Y_{t}+\eta_{t}s_{t}(Y_{t})\big)\quad(t=T,\ldots,1),\label{eq:prob-ode}
\end{equation}
where $\eta_{t}>0$ are the coefficients ensuring alignment between
the forward process (\ref{eq:forward-update}) and the reverse process.
Various choices of $\eta_{t}$ exist in the literature. For example,
\cite{li2023towards,li2024sharp} adopts the simple choice $\eta_{t}=(1-\alpha_{t})/2$.
In this paper, we adopt the coefficient design proposed in the original
DDIM paper \citep{song2020denoising}:
\begin{equation}
\eta_{t}^{\star}=1-\overline{\alpha}_{t}-\sqrt{(1-\overline{\alpha}_{t})(\alpha_{t}-\overline{\alpha}_{t})}.\label{eq:coefficient-design}
\end{equation}
This design enables the probability flow ODE sampler to achieve nearly
dimension-free convergence and adapt to unknown low-dimensional structures
in the target data distribution.

\paragraph*{Notation.}

The total variation (TV) distance between two probability distributions
$P$ and $Q$ on the same probability space $(\Omega,\mathcal{F})$ is define
as 
\begin{equation}
\mathsf{TV}(P,Q)\coloneqq \sup_{A\in\mathcal{F}}\vert P(A)-Q(A)\vert = \frac{1}{2}\int_{\Omega}\vert p(x)-q(x)\vert\mathrm{d}x,\label{eq:TV-defn}
\end{equation}
where the last relation holds when $P$ and $Q$ have density functions $p(x)$
and $q(x)$ with respect to the Lebesgue measure. For any matrix $A$,
we use $\Vert A\Vert$, $\Vert A\Vert_{\mathrm{F}}$, and $\mathsf{Tr}(A)$
to denote its spectral (operator) norm, Frobenius norm, and trace
(if $A$ is a square matrix), respectively. The support of $p_{\mathsf{data}}$,
denoted $\mathcal{X}\subseteq\mathbb{R}^{d}$, is the smallest closed
set $C\subseteq\mathbb{R}^{d}$ such that $p_{\mathsf{data}}(C)=1$. 

\section{Main results \protect\label{sec:main-results}}

In this section, we present our theoretical guarantees for the probability
flow ODE sampler. Before doing so, we introduce some necessary assumptions.
First, note that image data are inherently bounded, as pixel values
are typically represented within a fixed range, such as $[0,1]$ after
normalization. This bounded nature ensures that all valid images lie
within a compact subset of the data space. Motivated by this observation,
we impose the following assumption on the boundedness of the support
$\mathcal{X}=\mathsf{supp}(p_{\mathsf{data}})$ of the target distribution.

\begin{assumption}\label{assumption:bounded}There exists a universal
constant $c_{R}>0$ such that the radius of $\mathcal{X}$ is bounded
by $T^{c_{R}}$, namely:
\begin{equation}
\sup_{x\in\mathcal{X}}\left\Vert x\right\Vert _{2}\le T^{c_{R}}\label{Assumption=000020bounded}
\end{equation}
\end{assumption}

Since $c_{R}$ can be any fixed constant, regardless of how large,
this assumption allows the magnitude of the data to be large, making
it mild and practical. Moreover, it is a standard assumption in the
literature on diffusion model theory (e.g., \citet{huang2024convergence,li2024accelerating,li2024sharp}).
Next, following \citet{li2024adapting,huang2024denoising}, we define
the intrinsic dimension of $\mathcal{X}$ using the concept of covering
numbers or metric entropy.

\begin{definition}[Intrinsic dimension] \label{defn:intrinsic}Fix
$\varepsilon=T^{-c_{\varepsilon}}$ for some sufficiently large constant
$c_{\varepsilon}>0$.\footnote{The choice $\varepsilon=T^{-c_{\varepsilon}}$ simplifies notation and is not essential. In general, one can use any sufficiently small $\varepsilon>0$, replacing the $\log T$ term in \eqref{eq:assumption-low} with $\log(1/\varepsilon)$, though this introduces an additional parameter $\varepsilon$ in the final results.}
The intrinsic dimension of $\mathcal{X}$ is defined as the quantity
$k>0$ such that:
\begin{equation}
\log N_{\varepsilon}(\mathcal{X})\le C_{\mathsf{cover}}k\log T\label{eq:assumption-low}
\end{equation}
holds for some universal constant $C_{\mathsf{cover}}>0$. Here $N_{\varepsilon}(\mathcal{X})$
denotes the $\varepsilon$-covering number of the set $\mathcal{X}$
(see e.g., \cite[Definition 4.2.2]{vershynin2018high}). \end{definition}

This definition of intrinsic dimension is closely tied to the metric
entropy \citep{wainwright2019high}, which provides a natural measure
of a set\textquoteright s complexity in a metric space. Importantly,
it accommodates approximate low-dimensional structures, such as low-dimensional
manifolds, and is therefore more general than assuming that $\mathcal{X}$
resides in a linear subspace of $\mathbb{R}^{d}$. For a detailed
justification of this definition, see the discussion in \citet{huang2024denoising}.

With these assumptions in place, we are now ready to present our convergence
theory for the probability flow ODE sampler.

\begin{theorem}\label{thm:main} Consider the probability flow ODE
sampler (\ref{eq:prob-ode}) and take $\eta_{t}=\eta_{t}^{\star}$
(cf.~(\ref{eq:coefficient-design})). Then there exists some universal
constant $c>0$ such that 
\begin{equation}
\mathsf{TV}(p_{X_{1}},p_{Y_{1}})\leq c\frac{(k+\log d)\log^{3}T}{T}+c(\varepsilon_{\mathsf{score}}+\varepsilon_{\mathsf{Jacobi}})\log T,\label{eq:error-bound-low-d}
\end{equation}
where $k$ is the intrinsic dimension of $\mathcal{X}$ (see Definition~\ref{defn:intrinsic}).
\end{theorem}

The proof of Theorem~\ref{thm:main} is provided in Section~\ref{sec:proof-outline}.
Below, we discuss several implications of the result.
\begin{itemize}
\item \textit{Adaptivity to unknown low-dimensional structures.} Assume
access to perfect score estimation (i.e., $\varepsilon_{\mathsf{score}}=\varepsilon_{\mathsf{Jacobi}}=0$)
and ignore any logarithmic factors. Theorem~\ref{thm:main} establishes
that, with the coefficient design in (\ref{eq:coefficient-design}),
the probability flow ODE sampler achieves a convergence rate of $O(k/T)$
in TV distance. This demonstrates its capability for automatic adaptation
to unknown low-dimensional structures and dimension-free convergence.
Compared to the existing convergence theory in \citet{li2024sharp},
which yields a rate of $O(d/T)$, our result generalizes this to an
adaptive rate of $O(k/T)$ (notice that $k\lesssim d$ always holds).
Additionally, our result does not require the stringent assumption
$T\gtrsim d^{2}$ that is necessary in \citet{li2024sharp}.
\item \textit{Stability against score matching error.} The error bound (\ref{eq:error-bound-low-d})
includes two score matching error terms: $\varepsilon_{\mathsf{score}}$
and $\varepsilon_{\mathsf{Jacobi}}$. The term $\varepsilon_{\mathsf{score}}$
is standard and frequently appears in the diffusion model literature
(e.g.,~\citet{chen2022sampling,chen2022improved,li2023towards,li2024d,benton2023linear}).
However, as \citet{li2024sharp} demonstrated with a counterexample,
controlling $\varepsilon_{\mathsf{score}}$ alone is insufficient
to ensure convergence of the probability flow ODE sampler in TV distance;
the stricter Jacobian error $\varepsilon_{\mathsf{Jacobi}}$ also
need to be controlled. Importantly, our theory shows that both error
terms degrade gracefully as the accuracy of the learned score functions
decreases.
\end{itemize}
Finally, we compare our results with the concurrent work of \citet{liang2025low},
which analyzed the convergence of both stochastic and deterministic
samplers in the presence of low-dimensional structures. Their convergence
rate for the probability flow ODE sampler, stated in their Theorem
1, can be interpreted as:
\begin{equation}
\mathsf{TV}(p_{X_{1}},p_{Y_{1}})\leq c\frac{k\log^{3}T}{T}+c(\varepsilon_{\mathsf{score}}+\varepsilon_{\mathsf{Jacobi}}+\varepsilon_{\mathsf{Hess}})\sqrt{\log T},\label{eq:liang-et-al-bound}
\end{equation}
where $\varepsilon_{\mathsf{score}}$ and $\varepsilon_{\mathsf{Jacobi}}$
are defined similarly to our formulation (\ref{eq:score-error}),
except that their definition of $\varepsilon_{\mathsf{score},t}(x)$
does not include the inner product term $\big[(s_{t}(x)-s_{t}^{\star}(x))^{\top}s_{t}^{\star}(x)\big]^{2}$.
The additional term, $\varepsilon_{\mathsf{Hess}}$, measures the
score estimation error with respect to the Hessian matrix and is defined
as:
\[
\varepsilon_{\mathsf{Hess}}^{2}\coloneqq\frac{1}{T}\sum_{t=1}^{T}\mathbb{E}\big[\big\Vert\nabla\mathsf{Tr}\big(J_{s_{t}}(X_{t})-J_{s_{t}^{\star}}(X_{t})\big)\Vert_{2}^{2}\big].
\]
Under perfect score estimation, both our theory and \citet{liang2025low}
achieve the same convergence rate of $O(k/T)$. However, their analysis
requires $\varepsilon_{\mathsf{Hess}}$ to be small, implying that
the score functions must be accurately estimated up to their second
order derivatives, a condition not required by our theory. It is worth
noting that \citet{liang2025low} also studied the DDPM sampler, which
is beyond the scope of this paper.

\section{Proof of Theorem~\ref{thm:main}\protect\label{sec:proof-outline}}

In this section, we present the proof of Theorem~\ref{thm:main}.
We begin by recalling the update rule of the probability flow ODE
sampler (\ref{eq:prob-ode}) under the coefficient design (\ref{eq:coefficient-design}),
given by
\begin{equation}
Y_{t-1}=\frac{1}{\sqrt{\alpha_{t}}}\phi_{t}(Y_{t}):=\frac{1}{\sqrt{\alpha_{t}}}\big(Y_{t}+\eta_{t}^{\star}s_{t}(Y_{t})\big),\label{eq:updated=000020rule}
\end{equation}
where $\eta_{t}^{\star}=1-\overline{\alpha}_{t}-\sqrt{(1-\overline{\alpha}_{t})(\alpha_{t}-\overline{\alpha}_{t})}$
and the function $\phi_{t}(x)=x+\eta_{t}^{\star}s_{t}(x)$. Notably,
in the desired bound (\ref{eq:error-bound-low-d}), the terms $k$
and $\log d$ appear as a sum. Therefore, without loss of generality,
we assume $k\geq\log d$ throughout the proof. 

\subsection{Step 1: preliminaries}

Recall that $N_{\varepsilon}$ is the $\varepsilon$-covering number
of $\mathcal{X}=\mathsf{supp}(p_{\mathsf{data}})$. Our analysis requires
that $\varepsilon$ is sufficiently small, i.e.,
\begin{equation}
\varepsilon\ll\sqrt{\frac{1-\overline{\alpha}_{t}}{\overline{\alpha}_{t}}}\min\left\{ 1,\sqrt{\frac{k\log T}{d}}\right\} .\label{eq:eps-condition}
\end{equation}
Let $\{x_{i}^{\star}\}_{1\leq i\leq N_{\varepsilon}}$ be an $\varepsilon$-net
of $\mathcal{\mathcal{X}}$, and let $\{\mathcal{B}_{i}\}_{1\leq i\leq N_{\varepsilon}}$
be a disjoint $\varepsilon$-cover for $\mathcal{X}$ such that $x_{i}^{\star}\in\mathcal{B}_{i}$
(see e.g., \citet{vershynin2018high} for the definition). Following
the construction in \citet{li2024adapting}, we define the following
sets \begin{subequations}
\begin{align}
\mathcal{I} & \coloneqq\left\{ 1\leq i\leq N_{\varepsilon}:\mathbb{P}(X_{0}\in\mathcal{B}_{i})\geq\exp(-C_{1}k\log T)\right\} , \quad \text{and}\\
\mathcal{G} & \coloneqq\big\{\omega\in\mathbb{R}^{d}:\Vert\omega\Vert_{2}\leq2\sqrt{d}+\sqrt{C_{1}k\log T},\\
 & \qquad\qquad\qquad\vert(x_{i}^{\star}-x_{j}^{\star})^{\top}\omega\vert\leq\sqrt{C_{1}k\log T}\Vert x_{i}^{\star}-x_{j}^{\star}\Vert_{2},\,\,\forall\,1\leq i,j\leq N_{\varepsilon}\big\},
\end{align}
\end{subequations}where $C_{1}>0$ is some sufficiently large constant.
For each $1\leq t\leq T$, we define 
\begin{equation}
\mathcal{T}_{t}\coloneqq\left\{ \sqrt{\overline{\alpha}_{t}}x_{0}+\sqrt{1-\overline{\alpha}_{t}}\omega:x_{0}\in\cup_{i\in\mathcal{I}}\mathcal{B}_{i},\omega\in\mathcal{G}\right\} .\label{eq:=000020typical=000020set}
\end{equation}
In fact, $\mathcal{T}_{t}$ can be viewed as a high-probability set
for $X_{t}$. This is because $X_{t}$ can be expressed as $\sqrt{\overline{\alpha}_{t}}X_{0}+\sqrt{1-\overline{\alpha}_{t}}\,\overline{W}_{t}$
(see (\ref{eq:forward-formula})), and $\cup_{i\in\mathcal{I}}\mathcal{B}_{i}$
and $\mathcal{G}$ can be interpreted as high probability sets for
$X_{0}$ and $\overline{W}_{t}$, respectively. The following lemma
makes rigorous this intuition.

\begin{lemma} \label{lem:0} Suppose that $T$ is sufficient large
and that $c_{2}\gg2c_{R}+1$. Then for each $1\leq t\leq T,$ we have
\begin{equation}
\mathbb{P}\left(X_{t}\notin\mathcal{T}_{t}\right)\leq\exp\left(-\frac{C_{1}}{4}k\log T\right).\label{eq:0-1}
\end{equation}
In addition, we also have
\begin{equation}
\mathcal{\mathbb{P}}\left(Y_{T}\notin\mathcal{T}_{T}\right)\le T^{-99}+\exp\Big(-\frac{C_{1}}{4}k\log T\Big).\label{eq:0-2}
\end{equation}

\end{lemma}

\begin{proof} See Appendix~\ref{PF=0000200}.

\end{proof}

\subsection{Step 2: bounding the density ratio {\normalsize\textmd{$p_{Y_{t}}(y_{t})/p_{X_{t}}(y_{t})$}}}

For any $y_{T}\in\mathbb{R}^{d}$, we generate the trajectory $\{y_{T-1},y_{T-2}\dots,y_{1}\}$
from the deterministic reverse process (cf. \eqref{eq:updated=000020rule})
initialized from $y_{T}$:
\begin{equation}
y_{t-1}=\frac{1}{\sqrt{\alpha_{t}}}\phi_{t}(y_{t}),\qquad\text{for}\ t=T,\dots,2.\label{eq:trajectory}
\end{equation}
We aim to control the density ratio $\ensuremath{p_{Y_{t}}(y_{t})/p_{X_{t}}(y_{t})}$
when $y_{T}$ belongs to a certain set. First, as outlined in the
lemma below, $p_{Y_{T}}(y_{T})/p_{X_{T}}(y_{T})$ can be bounded up
to a universal constant whenever $y_{T}\in\mathcal{T}_{T}$. 

\begin{lemma} \label{lem-T} Suppose that $c_{2}\gg2c_{R}+1$. Then
for any $y_{T}\in\mathcal{T}_{T}$, we have
\[
p_{Y_{T}}(y_{T})\le e\cdot p_{X_{T}}(y_{T}).
\]

\end{lemma}

\begin{proof} See Appendix~\ref{PF=000020T}.\end{proof}

Next, to establish a connection between $p_{Y_{t}}(y_{t})/p_{X_{t}}(y_{t})$
and $p_{Y_{T}}(y_{T})/p_{X_{T}}(y_{T})$, we introduce the following
key lemma. This lemma reveals the connection between the density ratios
$p_{Y_{t}}(y_{t})/p_{X_{t}}(y_{t})$ and $p_{Y_{t-1}}(y_{t-1})/p_{X_{t-1}}(y_{t-1})$
for any $y_{t}\in\mathcal{T}_{t}$.

\begin{lemma} \label{lem:1} There exists some sufficient large constant
$C_{5}\ge64$ such that, for any $2\leq t\leq T$ and $x\in\mathcal{T}_{t}$
satisfying
\[
\frac{1-\alpha_{t}}{1-\overline{\alpha}_{t}}\varepsilon_{\mathsf{score},t}(x)\le1,\quad(1-\alpha_{t})\varepsilon_{\mathsf{Jacobi},t}(x)\le\frac{1}{8},\quad\text{and}\quad\left\Vert \sqrt{\frac{1-\overline{\alpha}_{t}}{\alpha_{t}-\overline{\alpha}_{t}}}\frac{\partial\phi_{t}^{\star}(x)}{\partial x}-I\right\Vert \le\frac{1}{8},
\]
 we have
\begin{align*}
 & \frac{p_{\sqrt{\alpha_{t}}X_{t-1}}(\phi_{t}(x))}{p_{X_{t}}(x)}/\frac{p_{\phi_{t}(Y_{t})}(\phi_{t}(x))}{p_{Y_{t}}(x)}\\
 & \qquad\ge\exp\bigg\{-C_{5}\bigg(\frac{k\log^{3}T}{T^{2}}+(1-\alpha_{t})\big[\varepsilon_{\mathsf{Jacobi},t}(x)+\varepsilon_{\mathsf{score},t}(x)\big]+\bigg\Vert\sqrt{\frac{1-\overline{\alpha}_{t}}{\alpha_{t}-\overline{\alpha}_{t}}}\frac{\partial\phi_{t}^{\star}(x)}{\partial x}-I\bigg\Vert_{\mathrm{F}}^{2}\bigg)\bigg\}.
\end{align*}
Here, the function $\phi_{t}^{\star}(x)$ is defined as
\begin{equation}
\phi_{t}^{\star}(x)=x+\eta_{t}^{\star}s_{t}^{*}(x)=x-\frac{\eta_{t}^{\star}}{1-\overline{\alpha}_{t}}\int_{x_{0}}\big(x-\sqrt{\overline{\alpha}_{t}}x_{0}\big)p_{X_{0}\mymid X_{t}}(x_{0}\mymid x)\mathrm{d}x_{0},\label{eq:phi}
\end{equation}
and the functions $\varepsilon_{\mathsf{Jacobi},t}(x)$ and $\varepsilon_{\mathsf{score},t}(x)$
are defined in (\ref{subeq:score-error}). 

\end{lemma}

\begin{proof} See Appendix~\ref{PF=0000201}.\end{proof} 

Now we define a sequence of sets $\mathcal{E}_{T},\ldots,\mathcal{E}_{2}$
as follows 
\begin{equation}
\mathcal{E}_{t}:=\big\{ y_{T}:y_{t}\in\mathcal{T}_{t},S_{t-1}(y_{T})\le1\big\}\quad\text{for}\quad2\leq t\leq T,\label{eq:=000020epsilon_t}
\end{equation}
where 
\begin{equation}
S_{t}(y_{T})\coloneqq\sum_{i>t}^{T}\xi_{i}(y_{i})\label{eq:S}
\end{equation}
with
\begin{align}
\xi_{i}(y_{i}):=C_{5}\left(\frac{k\log^{3}T}{T^{2}}+\frac{c_{1}\log T}{T}\left[\varepsilon_{\mathsf{score},i}(y_{i})+\varepsilon_{\mathsf{Jacobi},i}(y_{i})\right]+\bigg\|\sqrt{\frac{1-\overline{\alpha}_{i}}{\alpha_{i}-\overline{\alpha}_{i}}}\frac{\partial\phi_{i}^{\star}(y_{i})}{\partial x}-I\bigg\|_{\mathrm{F}}^{2}\right).\label{eq:S_t}
\end{align}
Here we should understand $y_{T-1},y_{T-2},\ldots,y_{1}$ as functions
of $y_{T}$ (as they are determined by $y_{T}$ through the deterministic
reverse process (\ref{eq:trajectory})). We can deduce that for any
$y_{T}\in\cap_{t>i}\mathcal{E}_{t}$, 
\begin{align}
\frac{p_{X_{i}}(y_{i})}{p_{Y_{i}}(y_{i})} & =\frac{p_{X_{T}}(y_{T})}{p_{Y_{T}}(y_{T})}\prod_{t=i+1}^{T}\left(\frac{p_{X_{t-1}}(y_{t-1})}{p_{X_{t}}(y_{t-1})}/\frac{p_{Y_{t-1}}(y_{t})}{p_{Y_{t}}(y_{t})}\right)\nonumber \\
 & \overset{\text{(i)}}{=}\frac{p_{X_{T}}(y_{T})}{p_{Y_{T}}(y_{T})}\prod_{t=i+1}^{T}\left(\frac{p_{\sqrt{\alpha_{t}}X_{t-1}}(\phi_{t}(y_{t}))}{p_{X_{t}}(y_{t})}/\frac{p_{\phi_{t}(Y_{t})}(\phi_{t}(y_{t}))}{p_{Y_{t}}(y_{t})}\right)\nonumber \\
 & \overset{\text{(ii)}}{\ge}\frac{p_{X_{T}}(y_{T})}{p_{Y_{T}}(y_{T})}\prod_{t=i+1}^{T}\exp\left(-\xi_{t}(y_{t})\right)\nonumber \\
 & \overset{\text{(iii)}}{=}\frac{p_{X_{T}}(y_{T})}{p_{Y_{T}}(y_{T})}\exp\left(-S_{i}(y_{T})\right)\overset{\text{(iv)}}{\ge}\frac{p_{X_{T}}(y_{T})}{p_{Y_{T}}(y_{T})}e^{-1}.\label{eq:XI/YI}
\end{align}
Here step~(i) follows from \eqref{eq:trajectory} and $p_{\sqrt{\alpha_{t}}X_{t-1}}(\sqrt{\alpha_{t}}y_{t-1})/p_{\sqrt{\alpha_{t}}Y_{t-1}}(\sqrt{\alpha_{t}}y_{t-1})=p_{X_{t-1}}(y_{t-1})/p_{Y_{t-1}}(y_{t})$;
step~(ii) follows from Lemma~\eqref{lem:1}, the definition of $\xi_{t}(y_{t})$
(cf.~\eqref{eq:S_t}), and the bound $1-\alpha_{t}\le c_{1}(\log T)/T$
(see Lemma~\ref{lem:T1}); step~(iii) uses the definition of $S_{i}(y_{T})$
(cf. \eqref{eq:S}); and step (iv) holds since $S_{i}(y_{T})\le1$
for any $y_{T}\in\cap_{t>i}\mathcal{E}_{t}$ (cf. \eqref{eq:=000020epsilon_t}).

Combining \eqref{eq:XI/YI} and Lemma \ref{lem-T}, we achieve the
following bounds on the density ratio needed for our subsequent analysis:
for any $y_{T}\in\cap_{t>j}\mathcal{E}_{t}$ such that $1\le j\le i<T$,
\begin{subequations}\label{subeq:induction}
\begin{align}
\frac{p_{X_{i}}(y_{i})}{p_{Y_{i}}(y_{i})} & \ge\frac{p_{X_{T}}(y_{T})}{p_{Y_{T}}(y_{T})}\exp\left(-S_{i}(y_{T})\right)\ge\frac{p_{X_{T}}(y_{T})}{p_{Y_{T}}(y_{T})}e^{-1},\quad\text{and}\label{eq:l1}\\
p_{Y_{i}}(y_{i}) & \le e^{2}p_{X_{i}}(y_{i}).\label{eq:l2}
\end{align}
\end{subequations}

\subsection{Step 3: bounding the TV distance}

We partition $\mathbb{R}^{d}$ into the union of disjoint sets as
follows: 
\begin{equation}
\mathbb{R}^{d}=\text{\ensuremath{\big(}}\cap_{t>1}^{T}\mathcal{E}_{t}\big)\cup\big(\cup_{i=2}^{T}\big(\mathcal{E}_{i}^{\mathrm{c}}\cap(\cap_{t>i}^{T}\mathcal{E}_{t})\big)\big).\label{eq:partition}
\end{equation}
where we define $\cap_{t>i}^{T}\mathcal{E}_{t}=\mathbb{R}^{d}$ for
$i=T$. In words, $\cap_{t>1}^{T}\mathcal{E}_{t}$ represents the
``ideal'' set where $p_{Y_{t}}(y_{t})/p_{X_{t}}(y_{t})$ can be
bounded by (\ref{subeq:induction}) for all $t\ge1$, if $y_{T}\in\cap_{t>1}^{T}\mathcal{E}_{t}$;
while for $y_{T}\in\mathcal{E}_{i}^{\mathrm{c}}\cap(\cap_{t>i}^{T}\mathcal{E}_{t})$,
such control for $p_{Y_{t}}(y_{t})/p_{X_{t}}(y_{t})$ holds for $t\geq i$
and may not hold for $t\leq i-1$. 

For simplicity, define $\mathcal{A}=\{y_{T}\in\mathbb{R}^{d}:p_{Y_{1}}(y_{1})>p_{X_{1}}(y_{1})\}$.
Motivated by the above decomposition, we can decompose the TV error
as follows: 
\begin{align}
\mathsf{TV}\big(p_{X_{1}},p_{Y_{1}}\big) & \overset{\text{(a)}}{=}\int_{y_{T}\in\mathcal{A}}\big(p_{Y_{1}}(y_{1})-p_{X_{1}}(y_{1})\big)\mathrm{d}y_{1}\nonumber \\
 & \overset{\text{(b)}}{=}\int_{y_{T}\in\mathcal{A}\cap(\cap_{t>1}^{T}\mathcal{E}_{t})}\big(p_{Y_{1}}(y_{1})-p_{X_{1}}(y_{1})\big)\mathrm{d}y_{1}+\sum_{i=2}^{T}\int_{y_{T}\in\mathcal{E}_{i}^{\mathrm{c}}\cap(\cap_{t>i}^{T}\mathcal{E}_{t})\cap\mathcal{A}}\big(p_{Y_{1}}(y_{1})-p_{X_{1}}(y_{1})\big)\mathrm{d}y_{1}\nonumber \\
 & \overset{\text{(c)}}{=}\underbrace{\mathbb{E}\left[\bigg(1-\frac{p_{X_{1}}(Y_{1})}{p_{Y_{1}}(Y_{1})}\bigg)\mathds{1}\big\{ Y_{T}\in\mathcal{A}\cap(\cap_{t>1}^{T}\mathcal{E}_{t})\big\}\right]}_{:=I_{1}}+\underbrace{\sum_{i=2}^{T}\mathbb{E}\bigg[\mathds{1}\big\{ Y_{T}\in\mathcal{E}_{i}^{\mathrm{c}}\cap(\cap_{t>i}^{T}\mathcal{E}_{t})\big\}\bigg]}_{:=I_{2}}.\label{eq:TV}
\end{align}
Here step (a) follows from the definition of TV distance (cf.~(\ref{eq:TV-defn}));
step (b) utilizes the partition (\ref{eq:partition}); while step
(c) follows from
\[
\int_{y_{T}\in\mathcal{E}_{i}^{\mathrm{c}}\cap(\cap_{t>i}^{T}\mathcal{E}_{t})\cap\mathcal{A}}\big(p_{Y_{1}}(y_{1})-p_{X_{1}}(y_{1})\big)\mathrm{d}y_{1}\le\int_{y_{T}\in\mathcal{E}_{i}^{\mathrm{c}}\cap(\cap_{t>i}^{T}\mathcal{E}_{t})}p_{Y_{1}}(y_{1})\mathrm{d}y_{1}=\mathbb{E}\Big[\mathds{1}\big\{ Y_{T}\in\mathcal{E}_{i}^{\mathrm{c}}\cap(\cap_{t>i}^{T}\mathcal{E}_{t})\big\}\Big],
\]
where the last relation holds since $Y_{1}$ is determined solely
by $Y_{T}$ through the deterministic update rule \eqref{eq:updated=000020rule}.
It then boils down to bounding $I_{1}$ and $I_{2}$.

We first derive an upper bound for $I_{1}$ as follows:
\begin{align*}
I_{1} & \overset{\text{(i)}}{\le}\mathbb{E}\left[\bigg(1-\frac{p_{X_{T}}(Y_{T})}{p_{Y_{T}}(Y_{T})}\exp\big(-S_{1}(Y_{T})\big)\bigg)\mathds{1}\big\{ Y_{T}\in\mathcal{A}\cap(\cap_{t>1}^{T}\mathcal{E}_{t})\big\}\right]\\
 & \overset{\text{(ii)}}{\le}\mathbb{E}\left[\bigg(1-\frac{p_{X_{T}}(Y_{T})}{p_{Y_{T}}(Y_{T})}\left(1-S_{1}(Y_{T})\right)\bigg)\mathds{1}\big\{ Y_{T}\in\mathcal{A}\cap(\cap_{t>1}^{T}\mathcal{E}_{t})\big\}\right]\\
 & =\underbrace{\int\left(p_{Y_{T}}(y_{T})-p_{X_{T}}(y_{T})\right)\mathds{1}\big\{ y_{T}\in\mathcal{A}\cap(\cap_{t>1}^{T}\mathcal{E}_{t})\big\}\mathrm{d}y_{T}}_{:=I_{1,1}}+\underbrace{\mathbb{E}\bigg[\mathds{1}\big\{ Y_{T}\in\mathcal{A}\cap(\cap_{t>1}^{T}\mathcal{E}_{t})\big\}\frac{p_{X_{T}}(Y_{T})}{p_{Y_{T}}(Y_{T})}\sum_{t=2}^{T}\xi_{t}(Y_{t})\bigg]}_{:=I_{1,2}}.
\end{align*}
Here step (i) follows from \eqref{eq:l1}, and step (ii) utilizes
the relation $\exp(-x)\ge1-x$. Note that throughout the above derivations,
the condition $Y_{T}\in\mathcal{A}\cap(\cap_{t>1}^{T}\mathcal{E}_{t})$
ensures that $p_{X_{1}}(Y_{1})/p_{Y_{1}}(Y_{1}),$ $p_{X_{t}}(Y_{t})/p_{Y_{t}}(Y_{t})$
and $p_{X_{T}}(Y_{T})/p_{Y_{T}}(Y_{T}$) are bounded, making them
well-defined in the expectations. In view of the definition of the
TV distance (cf. \eqref{eq:TV-defn}), we can upper bound $I_{1,1}$
as follows:
\[
I_{1,1}\le\int\left|p_{Y_{T}}(y_{T})-p_{X_{T}}(y_{T})\right|\mathrm{d}y_{T}=2\cdot\text{TV}(p_{X_{T}},p_{Y_{T}}).
\]
Regarding $I_{1,2}$, we have
\begin{align*}
I_{1,2}\overset{\text{(a)}}{\le} & \sum_{t=2}^{T}\mathbb{E}\left[e\frac{p_{X_{t}}(Y_{t})}{p_{Y_{t}}(Y_{t})}\xi_{t}(Y_{t})\mathds{1}\left\{ Y_{T}\in\mathcal{A}\cap(\cap_{i>1}^{T}\mathcal{E}_{i})\right\} \right]\\
= & \sum_{t=2}^{T}\int_{y_{T}\in\mathcal{A}\cap(\cap_{i>1}^{T}\mathcal{E}_{i})}e\frac{p_{X_{t}}(y_{t})}{p_{Y_{t}}(y_{t})}\xi_{t}(y_{t})\cdot p_{Y_{t}}(y_{t})\mathrm{d}y_{t}\le e\sum_{t=2}^{T}\mathbb{E}\left[\xi_{t}(X_{t})\right]
\end{align*}
where step (a) follows from the relation \eqref{eq:l1}. Combining
the above results, we arrive at the final bound for $I_{1}$:
\begin{equation}
I_{1}\le2\cdot\text{TV}(p_{X_{T}},p_{Y_{T}})+e\sum_{t=2}^{T}\mathbb{E}\left[\xi_{t}(X_{t})\right]\label{eq:11}
\end{equation}

To bound $I_{2},$ we need the following lemma. 

\begin{lemma}\label{lem:i2}The following relation holds true: 
\[
I_{2}\le e^{2}\sum_{t=2}^{T}\mathbb{E}\left[\xi_{t}(X_{t})\right]+e^{2}\sum_{t=2}^{T-1}\mathbb{P}\left(X_{t}\notin\mathcal{T}_{t}\right)+\mathcal{\mathbb{P}}\left(Y_{T}\notin\mathcal{T}_{T}\right).
\]

\end{lemma}

\begin{proof} See Appendix~\ref{PF=0000202}.\end{proof}

Finally, putting \eqref{eq:TV}, \eqref{eq:11}, and Lemma~\ref{lem:i2}
together, we conclude that
\begin{align*}
\mathsf{TV}\big(p_{X_{1}},p_{Y_{1}}\big) & \le2\cdot\mathsf{TV}(p_{X_{T}},p_{Y_{T}})+\left(e+e^{2}\right)\sum_{t=2}^{T}\mathbb{E}\left[\xi_{t}(X_{t})\right]+e^{2}\sum_{t=2}^{T-1}\mathbb{P}\left(X_{t}\notin\mathcal{T}_{t}\right)+\mathcal{\mathbb{P}}\left(Y_{T}\notin\mathcal{T}_{T}\right)\\
 & \overset{(\text{a})}{\le}3T^{-99}+11\sum_{t=2}^{T}\mathbb{E}\left[\xi_{t}(X_{t})\right]+8T\exp\Big(-\frac{C_{1}}{4}k\log T\Big)\\
 & \overset{(\text{b})}{\le}12C_{5}\Big(\frac{k\log^{3}T}{T}+c_{1}\big(\varepsilon_{\mathsf{score}}+\varepsilon_{\mathsf{Jacobi}}\big)\log T+\frac{C_{6}k\log^{2}T}{T}\Big)\\
 & \le c\Big(\frac{k\log^{3}T}{T}+\big(\varepsilon_{\mathsf{score}}+\varepsilon_{\mathsf{Jacobi}}\big)\log T\Big)
\end{align*}
for some constant $c\gg C_{5}\big(c_{1}+C_{6}\big).$ Here relation
(a) makes use of $\mathsf{\mathsf{TV}}\left(p_{X_{T}},p_{Y_{T}}\right)\leq T^{-99}$
(see Lemma~\ref{lemma:T3}) and Lemma~\ref{lem:0}; while step (b)
follows from the definition of $\xi_{t}(X_{t})$ (see \eqref{eq:S_t}), Lemma~\ref{lemma:T4} and the Cauchy-Schwarz inequality, provided that $C_{1}$ and $T$ are sufficiently
large.

\section{Discussion}

In this paper, we establish new theoretical guarantees for the probability
flow ODE sampler in the context of diffusion-based generative models.
We show that, with carefully chosen coefficients, the probability
flow ODE sampler achieves an adaptive convergence rate of $O(k/T)$,
where $k$ is the intrinsic dimension of the target distribution.
This result provides evidence that the probability flow ODE sampler
can automatically adapt to unknown low-dimensional structures in the
data distribution, leading to improved sampling efficiency. From a
technical perspective, we also develop a more concise analysis framework
compared to prior work (e.g., \citet{li2024sharp}).

Moving forward, there are several directions worthy of future investigation.
While our results establish a fast convergence rate for the probability
flow ODE sampler, it remains unclear whether this rate is optimal.
Investigating lower bounds for certain challenging instances would
provide deeper insights into the fundamental limits of SGMs. Additionally,
an open question is whether the coefficient design in (\ref{eq:coefficient-design})
is the unique choice that enables nearly dimension-free convergence
(in terms of total variation distance between the target and generated
distributions). Another promising avenue for future research is extending
our theoretical framework to establish sharp and adaptive convergence
rates for the probability flow ODE sampler under the Wasserstein distance.

\section*{Acknowledgements}

The authors thank Gen Li for helpful discussions.

\appendix

\section{Proof of auxiliary lemmas}

\subsection{Proof of Lemma~\ref{lem:0} \protect\label{PF=0000200}}

The proof of claim \eqref{eq:0-1} can be found in \citep[Lemma 1]{li2024adapting}.
In this section, we provide the proof of claim \eqref{eq:0-2}. We
have
\begin{align*}
\mathcal{\mathbb{P}}\left(Y_{T}\notin\mathcal{T}_{T}\right) & \overset{\text{(i)}}{\le}\mathcal{\mathbb{P}}\left(X_{T}\notin\mathcal{T}_{T}\right)+\left|\mathcal{\mathbb{P}}\left(Y_{T}\notin\mathcal{T}_{T}\right)-\mathcal{\mathbb{P}}\left(Y_{T}\notin\mathcal{T}_{T}\right)\right|\\
 & \overset{\text{(ii)}}{\le}\exp\left(-\frac{C_{1}}{4}k\log T\right)+\sup_{A\in\mathcal{F}}\left|\mathcal{\mathbb{P}}\left(Y_{T}\in A\right)-\mathcal{\mathbb{P}}\left(Y_{T}\in A\right)\right|\\
 & \overset{\text{(iii)}}{\le}\exp\left(-\frac{C_{1}}{4}k\log T\right)+\mathsf{TV}(p_{X_{T}},p_{Y_{T}})\\
 & \overset{\text{(iv)}}{\le}\exp\left(-\frac{C_{1}}{4}k\log T\right)+T^{-99},
\end{align*}
where step (i) makes use of triangle inequality; step (ii) applies
claim \eqref{eq:0-1}; step (iii) follows from the definition of TV
distance (see \ref{eq:TV-defn}); while step (iv) follows from Lemma~\ref{lemma:T3}.

\subsection{Proof of Lemma~\ref{lem-T} \protect\label{PF=000020T}}

Recalling the definition of $\mathcal{T}_{T}$ (see \eqref{eq:=000020typical=000020set}),
assume that $x=\sqrt{\overline{\alpha}_{T}}x_{0}+\sqrt{1-\overline{\alpha}_{T}}\omega\in\mathcal{T}_{T},$
where $\Vert\omega\Vert_{2}\leq2\sqrt{d}+\sqrt{C_{1}k\log T}$. We
have
\begin{align*}
\frac{p_{Y_{T}}(x)}{p_{X_{T}}(x)} & \overset{\text{(i)}}{=}\frac{\left(2\pi\right)^{d/2}\exp\left(-\Vert x\Vert_{2}^{2}/2\right)}{\int p_{X_{0}}(x_{0})\left[2\pi(1-\overline{\alpha}_{T})\right]^{-d/2}\exp\left(-\Vert x-\sqrt{\overline{\alpha}_{T}}x_{0}\Vert_{2}^{2}/\{2(1-\overline{\alpha}_{T})\}\right)\mathrm{d}x_{0}}\\
 & \overset{\text{(ii)}}{\le}\left(1-\overline{\alpha}_{T}\right)^{d/2}\exp\bigg\{\frac{1}{2\left(1-\overline{\alpha}_{T}\right)}\cdot\sup_{\Vert x_{0}\Vert_{2}\le T^{c_{R}}}\left(\Vert x-\sqrt{\overline{\alpha}_{T}}x_{0}\Vert_{2}^{2}-\left(1-\overline{\alpha}_{T}\right)\left\Vert x\right\Vert _{2}^{2}\right)\bigg\}\\
 & \overset{\text{(iii)}}{\le}\exp\left\{ \frac{1}{2\left(1-\overline{\alpha}_{T}\right)}\cdot\left(2T^{-c_{2}}\left(4d+2C_{1}k\log T\right)+4T^{-c_{2}/2+c_{R}}\left(2\sqrt{d}+\sqrt{C_{1}k\log T}\right)\right)\right\} \overset{(\text{iv})}{\le}e.
\end{align*}
Here, step (i) arises from $Y_{T}\sim\mathcal{N}(0,I_{d})$; step
(ii) utilizes Assumption \ref{Assumption=000020bounded}; step (iv)
holds true since $\overline{\alpha}_{T}\le T^{-c_{2}}$ (cf. Lemma
\ref{lem:T1}) as long as $c_{2}\gg2c_{R}+1$; and step (iii) follows
from the calculations below:
\begin{align*}
\sup_{\Vert x_{0}\Vert_{2}\le T^{c_{R}}}\left(\Vert x-\sqrt{\overline{\alpha}_{T}}x_{0}\Vert_{2}^{2}-\left(1-\overline{\alpha}_{T}\right)\left\Vert x\right\Vert _{2}^{2}\right) & =\sup_{\Vert x_{0}\Vert_{2}\le T^{c_{R}}}\left(\overline{\alpha}_{T}\left\Vert x\right\Vert _{2}^{2}+\overline{\alpha}_{T}\left\Vert x_{0}\right\Vert _{2}^{2}-2\sqrt{\overline{\alpha}_{T}}x^{\mathrm{T}}x_{0}\right)\\
 & \le\sup_{\Vert x_{0}\Vert_{2}\le T^{c_{R}}}\bigg(\overline{\alpha}_{T}\left\Vert x\right\Vert _{2}^{2}+\overline{\alpha}_{T}\left\Vert x_{0}\right\Vert _{2}^{2}+2\sqrt{\overline{\alpha}_{T}}\left\Vert x\right\Vert _{2}\left\Vert x_{0}\right\Vert _{2}\bigg)\\
 & \overset{\text{(a)}}{\le}T^{-c_{2}}\left\Vert x\right\Vert _{2}^{2}+T^{-c_{2}+2c_{R}}+2T^{-c_{2}/2+c_{R}}\left\Vert x\right\Vert _{2}\\
 & \overset{\text{(b)}}{\le}2T^{-c_{2}}\left(4d+2C_{1}k\log T\right)+4T^{-c_{2}/2+c_{R}}\left(2\sqrt{d}+\sqrt{C_{1}k\log T}\right),
\end{align*}
where step (a) applies $\overline{\alpha}_{T}\le T^{-c_{2}}$ (cf.
Lemma \ref{lem:T1}); and step (b) comes from 
\[
\Vert x\Vert_{2}\leq\sqrt{\overline{\alpha}_{T}}\Vert x_{0}\Vert_{2}+\sqrt{1-\overline{\alpha}_{T}}\Vert\omega\Vert_{2}\le T^{-c_{2}/2+c_{R}}+2\sqrt{d}+\sqrt{C_{1}k\log T}
\]
provided that $c_{2}\gg2c_{R}+1.$

\subsection{Proof of Lemma~\ref{lem:1} \protect\label{PF=0000201}}

In the following, we will analyze $p_{\sqrt{\alpha_{t}}X_{t-1}}(\phi_{t}(x))/p_{X_{t}}(x)$
and $p_{\phi_{t}(Y_{t})}(\phi_{t}(x))/p_{Y_{t}}(x)$ separately.

First, for $p_{\sqrt{\alpha_{t}}X_{t-1}}(\phi_{t}(x))/p_{X_{t}}(x)$,
we start with the following decomposition:
\begin{align}
\frac{\big\| x+\eta_{t}^{\star}s_{t}(x)-\sqrt{\overline{\alpha}_{t}}x_{0}\big\|_{2}^{2}}{2(\alpha_{t}-\overline{\alpha}_{t})} & =\frac{\big\| x-\sqrt{\overline{\alpha}_{t}}x_{0}\big\|_{2}^{2}}{2(1-\overline{\alpha}_{t})}+\frac{(1-\alpha_{t})\big\| x-\sqrt{\overline{\alpha}_{t}}x_{0}\big\|_{2}^{2}}{2(\alpha_{t}-\overline{\alpha}_{t})(1-\overline{\alpha}_{t})}+\frac{2\eta_{t}^{\star}s_{t}(x)^{\top}(x-\sqrt{\overline{\alpha}_{t}}x_{0})+\|\eta_{t}^{\star}s_{t}(x)\|_{2}^{2}}{2(\alpha_{t}-\overline{\alpha}_{t})}\nonumber \\
 & =\frac{\big\| x-\sqrt{\overline{\alpha}_{t}}x_{0}\big\|_{2}^{2}}{2(1-\overline{\alpha}_{t})}+\frac{(1-\alpha_{t})\overline{\alpha}_{t}\big\| x_{0}-\widehat{x}_{0}\big\|_{2}^{2}}{2(\alpha_{t}-\overline{\alpha}_{t})(1-\overline{\alpha}_{t})}+\zeta_{t}(x,x_{0}),\label{eq:1}
\end{align}
where we define
\[
\widehat{x}_{0}:=\int p_{X_{0}|X_{t}}(x_{0}\mymid x)x_{0}\mathrm{d}x_{0},
\]
and
\begin{align*}
\zeta_{t}(x,x_{0}) & :=-\frac{\eta_{t}^{\star}\sqrt{\overline{\alpha}_{t}}\big(x_{0}-\widehat{x}_{0}\big)^{\top}(x-\sqrt{\overline{\alpha}_{t}}\widehat{x}_{0})}{\sqrt{(\alpha_{t}-\overline{\alpha}_{t})(1-\overline{\alpha}_{t})^{3}}}+\frac{2\eta_{t}^{\star}\big[s_{t}(x)-s_{t}^{\star}(x)\big]^{\top}(x-\sqrt{\overline{\alpha}_{t}}\widehat{x}_{0})}{2\sqrt{(\alpha_{t}-\overline{\alpha}_{t})(1-\overline{\alpha}_{t})}}\\
 & +\frac{(\eta_{t}^{\star})^{2}\big\| s_{t}(x)-s_{t}^{\star}(x)\big\|_{2}^{2}}{2(\alpha_{t}-\overline{\alpha}_{t})}-\left(\frac{1-\overline{\alpha}_{t}}{\alpha_{t}-\overline{\alpha}_{t}}-\sqrt{\frac{1-\overline{\alpha}_{t}}{\alpha_{t}-\overline{\alpha}_{t}}}\right)\big(s_{t}(x)-s_{t}^{\star}(x)\big)^{\top}\sqrt{\overline{\alpha}_{t}}\big(x_{0}-\widehat{x}_{0}\big).
\end{align*}
Recall from (\ref{eq:forward-formula}) that $X_{t}\mymid X_{0}=x_{0}\sim\mathcal{N}(\sqrt{\overline{\alpha}_{t}}x_{0},(1-\overline{\alpha}_{t})I_{d})$.
We can write the conditional distribution of $X_{0}$ given $X_{t}=x$
as
\begin{equation}
p_{X_{0}|X_{t}}(x_{0}\mymid x)=\frac{p_{X_{0}}(x_{0})}{p_{X_{t}}(x)}p_{X_{t}|X_{0}}(x\mymid x_{0})=\frac{p_{X_{0}}(x_{0})}{p_{X_{t}}(x)}\left[2\pi(1-\overline{\alpha}_{t})\right]^{-d/2}\exp\left(-\frac{\Vert x-\sqrt{\overline{\alpha}_{t}}x_{0}\Vert_{2}^{2}}{2(1-\overline{\alpha}_{t})}\right).\label{eq:2}
\end{equation}
Then, using the update rule \eqref{eq:updated=000020rule}, we have
\begin{align}
 & \frac{p_{\sqrt{\alpha_{t}}X_{t-1}}(\phi_{t}(x)\big)}{p_{X_{t}}(x)}=\frac{1}{p_{X_{t}}(x)}\int_{x_{0}}p_{x_{0}}(x_{0})p_{\sqrt{\alpha_{t}}X_{t-1}|X_{0}}(\phi_{t}(x))\mathrm{d}x_{0}\nonumber\\
 & \quad\overset{\text{(i)}}{=}\frac{1}{p_{X_{t}}(x)}\int_{x_{0}}p_{x_{0}}(x_{0})[2\pi(\alpha_{t}-\overline{\alpha}_{t})]^{-d/2}\exp\left(-\frac{\Vert\phi_{t}(x)-\sqrt{\overline{\alpha}_{t}}x_{0}\Vert_{2}^{2}}{2(\alpha_{t}-\overline{\alpha}_{t})}\right)\mathrm{d}x_{0}\nonumber \\
 & \quad\overset{\text{(ii)}}{=}\left(\frac{1-\overline{\alpha}_{t}}{\alpha_{t}-\overline{\alpha}_{t}}\right)^{d/2}\int_{x_{0}}p_{X_{0}|X_{t}}(x_{0}\mymid x)\exp\left(-\frac{\Vert\phi_{t}(x)-\sqrt{\overline{\alpha}_{t}}x_{0}\Vert_{2}^{2}}{2(\alpha_{t}-\overline{\alpha}_{t})}+\frac{\Vert x-\sqrt{\overline{\alpha}_{t}}x_{0}\Vert_{2}^{2}}{2(1-\overline{\alpha}_{t})}\right)\mathrm{d}x_{0}\nonumber \\
 & \quad\overset{\text{(iii)}}{=}\left(\frac{1-\overline{\alpha}_{t}}{\alpha_{t}-\overline{\alpha}_{t}}\right)^{d/2}\int_{x_{0}}p_{X_{0}|X_{t}}(x_{0}\mymid x)\exp\left(-\frac{(1-\alpha_{t})\overline{\alpha}_{t}\big\| x_{0}-\widehat{x}_{0}\big\|_{2}^{2}}{2(\alpha_{t}-\overline{\alpha}_{t})(1-\overline{\alpha}_{t})}-\zeta_{t}(x,x_{0})\right)\mathrm{d}x_{0}\nonumber \\
 & \quad\overset{\text{(iv)}}{\ge}\left(\frac{1-\overline{\alpha}_{t}}{\alpha_{t}-\overline{\alpha}_{t}}\right)^{d/2}\exp\left(-\frac{(1-\alpha_{t})\overline{\alpha}_{t}\int p_{X_{0}|X_{t}}(x_{0}\mymid x)\big\| x_{0}-\widehat{x}_{0}\big\|_{2}^{2}\mathrm{d}x_{0}}{2(\alpha_{t}-\overline{\alpha}_{t})(1-\overline{\alpha}_{t})}-\int p_{X_{0}|X_{t}}(x_{0}\mymid x)\zeta_{t}(x,x_{0})\mathrm{d}x_{0}\right).\label{eq:4.0}
\end{align}
Here step (i) utilizes $\sqrt{\alpha_{t}}X_{t-1}\mymid X_{0}=x_{0}\sim\mathcal{N}(\sqrt{\overline{\alpha}_{t}}x_{0},(\alpha_{t}-\overline{\alpha}_{t})I_{d})$;
step (ii) makes use of \eqref{eq:2}; step (iii) follows from \eqref{eq:1};
while step (iv) applies Jensen's inequality. Before bounding the term
$\int p_{X_{0}|X_{t}}(x_{0}\mymid x)\zeta_{t}(x,x_{0})\mathrm{d}x_{0}$,
we observe that for $2\le t\le T,$ our choice of coefficient (\ref{eq:coefficient-design})
satisfies
\[
\eta_{t}^{\star}=1-\overline{\alpha}_{t}-\sqrt{(1-\overline{\alpha}_{t})(\alpha_{t}-\overline{\alpha}_{t})}=(1-\overline{\alpha}_{t})\left(1-\sqrt{1-\frac{1-\alpha_{t}}{1-\overline{\alpha}_{t}}}\right).
\]
Using the property of $\alpha_{t}$ (see Lemma \ref{lem:T1}) and
the inequality $1+x/2\ge\sqrt{1+x}\ge1+x/2-x^{2}$ for $x\ge-1$,
we derive the following relationship:
\begin{equation}
0\le\eta_{t}^{\star}-\frac{1-\alpha_{t}}{2}\le(1-\alpha_{t})\cdot\frac{1-\alpha_{t}}{1-\overline{\alpha}_{t}}\le\frac{1-\alpha_{t}}{2}\cdot\frac{16c_{1}\log T}{T}.\label{eq:0}
\end{equation}
With this result, we can proceed our analysis as follows:
\begin{align}
 & \bigg|\int p_{X_{0}|X_{t}}(x_{0}\mymid x)\zeta_{t}(x,x_{0})\mathrm{d}x_{0}\bigg|\overset{\text{(a)}}{\le}\frac{2\eta_{t}^{\star}\left|\big[s_{t}(x)-s_{t}^{\star}(x)\big]^{\top}(x-\sqrt{\overline{\alpha}_{t}}\widehat{x}_{0})\right|}{2\sqrt{(\alpha_{t}-\overline{\alpha}_{t})(1-\overline{\alpha}_{t})}}+\frac{(\eta_{t}^{\star})^{2}\big\| s_{t}(x)-s_{t}^{\star}(x)\big\|_{2}^{2}}{2(\alpha_{t}-\overline{\alpha}_{t})}\nonumber \\
 & \qquad\le\eta_{t}^{\star}\sqrt{\frac{1-\overline{\alpha}_{t}}{\alpha_{t}-\overline{\alpha}_{t}}}\varepsilon_{\mathsf{score},t}(x)+\frac{(\eta_{t}^{\star})^{2}\varepsilon_{\mathsf{score},t}(x)^{2}}{2(\alpha_{t}-\overline{\alpha}_{t})}\nonumber \\
 & \qquad\overset{\text{(b)}}{\le}\frac{1-\alpha_{t}}{2}\left(1+\frac{16c_{1}\log T}{T}\right)\sqrt{\frac{1-\overline{\alpha}_{t}}{\alpha_{t}-\overline{\alpha}_{t}}}\varepsilon_{\mathsf{score},t}(x)+\bigg(\frac{1-\alpha_{t}}{2}\bigg)^{2}\left(1+\frac{16c_{1}\log T}{T}\right)^{2}\frac{\varepsilon_{\mathsf{score},t}(x)^{2}}{2(\alpha_{t}-\overline{\alpha}_{t})}\nonumber \\
 & \qquad\overset{\text{(c)}}{\le}(1-\alpha_{t})\varepsilon_{\mathsf{score},t}(x)+(1-\alpha_{t})\varepsilon_{\mathsf{score},t}(x)^{2}\frac{1-\alpha_{t}}{1-\overline{\alpha}_{t}}\overset{\text{(d)}}{\le}5(1-\alpha_{t})\varepsilon_{\mathsf{score},t}(x).\label{eq:4}
\end{align}
Here, step (a) holds since $\int_{x_{0}}p_{X_{0}|X_{t}}(x_{0}\mymid x)(x_{0}-\widehat{x}_{0})\mathrm{d}x_{0}=0$;
step (b) follows from \eqref{eq:0}; while step (c) and (d) follows
from Lemma~\ref{lem:T1}, provided that $T$ is sufficient large
and $\frac{1-\alpha_{t}}{1-\overline{\alpha}_{t}}\varepsilon_{\mathsf{score},t}(x)\le1$.
Combining \eqref{eq:4} and \eqref{eq:4.0} yields
\begin{equation}
\frac{p_{\sqrt{\alpha_{t}}X_{t-1}}(\phi_{t}(x)\big)}{p_{X_{t}}(x)}\ge\left(\frac{1-\overline{\alpha}_{t}}{\alpha_{t}-\overline{\alpha}_{t}}\right)^{d/2}\exp\bigg(-\frac{(1-\alpha_{t})\overline{\alpha}_{t}\int p_{X_{0}|X_{t}}(x_{0}\mymid x)\big\| x_{0}-\widehat{x}_{0}\big\|_{2}^{2}\mathrm{d}x_{0}}{2(\alpha_{t}-\overline{\alpha}_{t})(1-\overline{\alpha}_{t})}-5(1-\alpha_{t})\varepsilon_{\mathsf{score},t}(x)\bigg)\label{eq:6}
\end{equation}

Next, for $p_{\phi_{t}(Y_{t})}(\phi_{t}(x))/p_{Y_{t}}(x),$ we have
the following transformation: 
\begin{align}
 & \frac{p_{\phi_{t}(Y_{t})}(\phi_{t}(x))}{p_{Y_{t}}(x)}=\left|\det\left(\frac{\partial\phi_{t}(x)}{\partial x}\right)^{-1}\right|=\left|\det\left(\frac{\partial\phi_{t}^{\star}(x)}{\partial x}+\bigg(\frac{\partial\phi_{t}(x)}{\partial x}-\frac{\partial\phi_{t}^{\star}(x)}{\partial x}\bigg)\right)^{-1}\right|\nonumber \\
 & \qquad=\left|\mathsf{\det}\bigg(\sqrt{\frac{\alpha_{t}-\overline{\alpha}_{t}}{1-\overline{\alpha}_{t}}}I+\frac{\partial\phi_{t}^{\star}(x)}{\partial x}-\sqrt{\frac{\alpha_{t}-\overline{\alpha}_{t}}{1-\overline{\alpha}_{t}}}I+\eta_{t}^{\star}\Big[J_{s_{t}}(x)-J_{s_{t}^{\star}}(x)\Big]\bigg)^{-1}\right|\nonumber \\
 & \qquad=\left(\frac{\alpha_{t}-\overline{\alpha}_{t}}{1-\overline{\alpha}_{t}}\right)^{-d/2}\left|\det\bigg(I+\bigg[\sqrt{\frac{1-\overline{\alpha}_{t}}{\alpha_{t}-\overline{\alpha}_{t}}}\frac{\partial\phi_{t}^{\star}(x)}{\partial x}-I\bigg]+\sqrt{\frac{1-\overline{\alpha}_{t}}{\alpha_{t}-\overline{\alpha}_{t}}}\eta_{t}^{\star}\Big[J_{s_{t}}(x)-J_{s_{t}^{\star}}(x)\Big]\bigg)^{-1}\right|,\label{eq:7}
\end{align}
where $\frac{\partial\phi_{t}^{\star}(x)}{\partial x}$ and $\frac{\partial\phi_{t}(x)}{\partial x}$
are the Jacobian matrices of $\phi_{t}^{\star}(x)$ and $\phi_{t}(x)$.
Using the expression of $\phi_{t}^{\star}$ (see \eqref{eq:phi}),
we can show that
\begin{align}
\sqrt{\frac{1-\overline{\alpha}_{t}}{\alpha_{t}-\overline{\alpha}_{t}}}\frac{\partial\phi_{t}^{\star}(x)}{\partial x}-I & =\bigg(\sqrt{\frac{1-\overline{\alpha}_{t}}{\alpha_{t}-\overline{\alpha}_{t}}}-1\bigg)\frac{\overline{\alpha}_{t}}{1-\overline{\alpha}_{t}}\bigg\{\int_{x_{0}}p_{X_{0}|X_{t}}(x_{0}\mymid x)x_{0}x_{0}^{\top}\mathrm{d}x_{0}\nonumber \\
 & \qquad\qquad-\Big(\int_{x_{0}}p_{X_{0}|X_{t}}(x_{0}\mymid x)x_{0}\mathrm{d}x_{0}\Big)\Big(\int_{x_{0}}p_{X_{0}|X_{t}}(x_{0}\mymid x)x_{0}\mathrm{d}x_{0}\Big)^{\top}\bigg\},\label{eq:8}
\end{align}
which demonstrates that
\begin{align}
 & \mathsf{Tr}\left(\sqrt{\frac{1-\overline{\alpha}_{t}}{\alpha_{t}-\overline{\alpha}_{t}}}\frac{\partial\phi_{t}^{\star}(x)}{\partial x}-I\right)=\bigg(\sqrt{\frac{1-\overline{\alpha}_{t}}{\alpha_{t}-\overline{\alpha}_{t}}}-1\bigg)\frac{\overline{\alpha}_{t}}{1-\overline{\alpha}_{t}}\bigg\{\int_{x_{0}}p_{X_{0}|X_{t}}(x_{0}\mymid x)\Vert x_{0}-\widehat{x}_{0}\Vert_{2}^{2}\mathrm{d}x_{0}\bigg\}.\label{eq:7.2}
\end{align}
Moreover, using \eqref{eq:0} and Lemma \ref{lem:T1}, one can bound
the term $\sqrt{\frac{1-\overline{\alpha}_{t}}{\alpha_{t}-\overline{\alpha}_{t}}}\eta_{t}^{\star}\Big[J_{s_{t}}(x)-J_{s_{t}^{\star}}(x)\Big]$
as follows:
\begin{align}
\left|\mathsf{Tr}\bigg(\sqrt{\frac{1-\overline{\alpha}_{t}}{\alpha_{t}-\overline{\alpha}_{t}}}\eta_{t}^{\star}\Big[J_{s_{t}}(x)-J_{s_{t}^{\star}}(x)\Big]\bigg)\right| & \le(1-\alpha_{t})\big|\mathsf{Tr}\big[J_{s_{t}}(x)-J_{s_{t}^{\star}}(x)\big]\big|\le(1-\alpha_{t})\varepsilon_{\mathsf{Jacobi},t}(x),\label{eq:7.3}\\
\left\Vert \sqrt{\frac{1-\overline{\alpha}_{t}}{\alpha_{t}-\overline{\alpha}_{t}}}\eta_{t}^{\star}\Big[J_{s_{t}}(x)-J_{s_{t}^{\star}}(x)\Big]\right\Vert _{\mathrm{F}}^{2} & \le(1-\alpha_{t})^{2}\bigg\Vert J_{s_{t}}(x)-J_{s_{t}^{\star}}(x)\bigg\Vert_{\mathrm{F}}^{2}\le(1-\alpha_{t})^{2}\varepsilon_{\mathsf{Jacobi},t}(x)^{2}\label{eq:7.4}
\end{align}
as long as $T\gg\log T.$ To proceed further, we use Lemma~\ref{lemma:logdet}
to achieve the following result: for any matrices $A,B\in\mathbb{R}^{d\times d}$
such that $\Vert A\Vert\le1/8$ and $\Vert B\Vert\le1/8,$ we have
\begin{align}
\log\left|\det(I+A+B)\right| & \overset{\text{(a)}}{\ge}\mathsf{Tr}(A)+\mathsf{Tr}(B)-2\left\Vert A+B\right\Vert _{\mathrm{F}}^{2}\nonumber \\
 & \overset{\text{(b)}}{\ge}\mathsf{Tr}(A)+\mathsf{Tr}(B)-4\left\Vert A\right\Vert _{\mathrm{F}}^{2}-4\left\Vert B\right\Vert _{\mathrm{\mathrm{F}}}^{2},\label{eq:7.0}
\end{align}
Here step (a) utilizes Lemma~\ref{lemma:logdet} and step (b) follows
from the Cauchy-Schwarz inequality. Combining the above results, we
deduce that
\begin{align}
 & \quad\ \log\frac{p_{\phi_{t}(Y_{t})}(\phi_{t}(x))}{p_{Y_{t}}(x)}\nonumber \\
 & \overset{\text{(i)}}{=}-\frac{d}{2}\log\left(\frac{\alpha_{t}-\overline{\alpha}_{t}}{1-\overline{\alpha}_{t}}\right)-\log\left|\mathsf{det}\bigg(I+\bigg[\sqrt{\frac{1-\overline{\alpha}_{t}}{\alpha_{t}-\overline{\alpha}_{t}}}\frac{\partial\phi_{t}^{\star}(x)}{\partial x}-I\bigg]+\sqrt{\frac{1-\overline{\alpha}_{t}}{\alpha_{t}-\overline{\alpha}_{t}}}\eta_{t}^{\star}\Big[J_{s_{t}}(x)-J_{s_{t}^{\star}}(x)\Big]\bigg)\right|\nonumber \\
 & \overset{\text{(ii)}}{\le}-\frac{d}{2}\log\left(\frac{\alpha_{t}-\overline{\alpha}_{t}}{1-\overline{\alpha}_{t}}\right)-\mathsf{Tr}\left(\sqrt{\frac{1-\overline{\alpha}_{t}}{\alpha_{t}-\overline{\alpha}_{t}}}\frac{\partial\phi_{t}^{\star}(x)}{\partial x}-I\right)-\mathsf{Tr}\left(\sqrt{\frac{1-\overline{\alpha}_{t}}{\alpha_{t}-\overline{\alpha}_{t}}}\eta_{t}^{\star}\Big[J_{s_{t}}(x)-J_{s_{t}^{\star}}(x)\Big]\right)\nonumber \\
 & \qquad\qquad+4\left\Vert \sqrt{\frac{1-\overline{\alpha}_{t}}{\alpha_{t}-\overline{\alpha}_{t}}}\frac{\partial\phi_{t}^{\star}(x)}{\partial x}-I\right\Vert _{\mathrm{F}}^{2}+4\left\Vert \sqrt{\frac{1-\overline{\alpha}_{t}}{\alpha_{t}-\overline{\alpha}_{t}}}\eta_{t}^{\star}\Big[J_{s_{t}}(x)-J_{s_{t}^{\star}}(x)\Big]\right\Vert _{\mathrm{F}}^{2}\nonumber \\
 & \overset{\text{(iii)}}{\le}-\frac{d}{2}\log\left(\frac{\alpha_{t}-\overline{\alpha}_{t}}{1-\overline{\alpha}_{t}}\right)-\mathsf{Tr}\bigg(\sqrt{\frac{1-\overline{\alpha}_{t}}{\alpha_{t}-\overline{\alpha}_{t}}}\frac{\partial\phi_{t}^{\star}(x)}{\partial x}-I\bigg)+4\left\Vert \sqrt{\frac{1-\overline{\alpha}_{t}}{\alpha_{t}-\overline{\alpha}_{t}}}\frac{\partial\phi_{t}^{\star}(x)}{\partial x}-I\right\Vert _{\mathrm{F}}^{2}\nonumber \\
 & \qquad\qquad+(1-\alpha_{t})\varepsilon_{\mathsf{Jacobi},t}(x)+4(1-\alpha_{t})^{2}\varepsilon_{\mathsf{Jacobi},t}(x)^{2}\nonumber \\
 & \overset{(\text{iv})}{\le}-\frac{d}{2}\log\left(\frac{\alpha_{t}-\overline{\alpha}_{t}}{1-\overline{\alpha}_{t}}\right)-\left(\sqrt{\frac{1-\overline{\alpha}_{t}}{\alpha_{t}-\overline{\alpha}_{t}}}-1\right)\frac{\overline{\alpha}_{t}}{1-\overline{\alpha}_{t}}\left\{ \int_{x_{0}}p_{X_{0}|X_{t}}(x_{0}\mymid x)\Vert x_{0}-\widehat{x}_{0}\Vert_{2}^{2}\mathrm{d}x_{0}\right\} \nonumber \\
 & \qquad\qquad+4\left\Vert \sqrt{\frac{1-\overline{\alpha}_{t}}{\alpha_{t}-\overline{\alpha}_{t}}}\frac{\partial\phi_{t}^{\star}(x)}{\partial x}-I\right\Vert _{\mathrm{F}}^{2}+5(1-\alpha_{t})\varepsilon_{\mathsf{Jacobi},t}(x),\label{eq:13}
\end{align}
where step (i) follows from \eqref{eq:7}; step (ii) follows from
\eqref{eq:7.0}; step (iii) follows from \eqref{eq:7.3} and \eqref{eq:7.4};
and step (iv) holds given that $(1-\alpha_{t})\varepsilon_{\mathsf{Jacobi},t}(x)\le1.$

Now, combining the relations \eqref{eq:6} and \eqref{eq:13} yields
\begin{align}
 & \frac{p_{\sqrt{\alpha_{t}}X_{t-1}}(\phi_{t}(x))}{p_{X_{t}}(x)}/\frac{p_{\phi_{t}(Y_{t})}(\phi_{t}(x))}{p_{Y_{t}}(x)}\nonumber \\
 & \qquad\ge\exp\Bigg\{\underbrace{\left(\sqrt{\frac{1-\overline{\alpha}_{t}}{\alpha_{t}-\overline{\alpha}_{t}}}-1-\frac{1-\alpha_{t}}{2(\alpha_{t}-\overline{\alpha}_{t})}\right)\frac{\overline{\alpha}_{t}}{1-\overline{\alpha}_{t}}\left(\int p_{X_{0}|X_{t}}(x_{0}\mymid x)\Vert x_{0}-\widehat{x}_{0}\Vert_{2}^{2}\mathrm{d}x_{0}\right)}_{:=\Delta}\nonumber \\
 & \qquad\qquad\qquad-5(1-\alpha_{t})\left[\varepsilon_{\mathsf{score},t}(x)+\varepsilon_{\mathsf{Jacobi},t}(x)\right]-4\left\Vert \sqrt{\frac{1-\overline{\alpha}_{t}}{\alpha_{t}-\overline{\alpha}_{t}}}\frac{\partial\phi_{t}^{\star}(x)}{\partial x}-I\right\Vert _{\mathrm{F}}^{2}\Bigg\}.\label{eq:14}
\end{align}
The remaining task is to bound $\Delta.$ Using Lemma \ref{lem:T1}
and the inequality $1+x/2\ge\sqrt{1+x}\ge1+x/2-x^{2}$ for $x\ge-1$
again, we obtain that
\begin{align}
0 & \ge\sqrt{\frac{1-\overline{\alpha}_{t}}{\alpha_{t}-\overline{\alpha}_{t}}}-1-\frac{1-\alpha_{t}}{2(\alpha_{t}-\overline{\alpha}_{t})}=\sqrt{1+\frac{1-\alpha_{t}}{\alpha_{t}-\overline{\alpha}_{t}}}-1-\frac{1-\alpha_{t}}{2(\alpha_{t}-\overline{\alpha}_{t})}\nonumber \\
 & \ge-\bigg(\frac{1-\alpha_{t}}{\alpha-\overline{\alpha}_{t}}\bigg)^{2}\ge-\frac{64c_{1}^{2}\log^{2}T}{T^{2}}.\label{eq:14.0}
\end{align}
In order to control the term
\[
\frac{\overline{\alpha}_{t}}{1-\overline{\alpha}_{t}}\left(\int p_{X_{0}|X_{t}}(x_{0}\mymid x)\Vert x_{0}-\widehat{x}_{0}\Vert_{2}^{2}\mathrm{d}x_{0}\right),
\]
we adopt some notation from \citet{li2024adapting}. The definition
of $\mathcal{T}_{t}$ (see \eqref{eq:=000020typical=000020set}) implies
that for any $x_{t}\in\mathcal{T}_{t}$, there exists an index $i(x_{t})\in\mathcal{I}$ such that we can find two points $x_{0}(x_{t})\in\mathcal{B}_{i(x_{t})}$ and $\omega\in\mathcal{G}$ satisfying
\[
x_{t}=\sqrt{\overline{\alpha}_{t}}x_{0}(x_{t})+\sqrt{1-\overline{\alpha}_{t}}\omega.
\]
For any constant $r>0$, we define 
\[
\mathcal{I}\left(x_{t};r\right)\coloneqq\left\{ 1\leq i\leq N_{\varepsilon}:\overline{\alpha}_{t}\Vert x_{i}^{\star}-x_{i(x_{t})}^{\star}\Vert_{2}^{2}\leq r\cdot k(1-\overline{\alpha}_{t})\log T\right\} .
\]
Then we fix some large enough constant $C_{3}>0$ and define two disjoint
sets as follows:
\begin{align*}
\mathcal{X}_{t}\left(x_{t}\right) & \coloneqq\bigcup_{i\in\mathcal{I}(x_{t};C_{3})}\mathcal{B}_{i}\qquad\text{and}\qquad\mathcal{Y}_{t}\left(x_{t}\right)\coloneqq\bigcup_{i\notin\mathcal{I}(x_{t};C_{3})}\mathcal{B}_{i}.
\end{align*}
Using the above notation, the following results hold as long as $x_{t}\in\mathcal{T}_{t}$
(see \citep[(A.32) and (A.15)]{li2024adapting}). Suppose that $x_{0}\in\mathcal{B}_{i}$
for some $1\leq i\leq N_{\varepsilon}$, then
\[
\Vert x_{0}-\widehat{x}_{0}\Vert_{2}\le\Vert x_{i(x_{i})}^{\star}-x_{i}^{\star}\Vert_{2}+4\sqrt{\frac{C_{3}k(1-\overline{\alpha}_{t})\log T}{\overline{\alpha}_{t}}}.
\]
Furthermore, if $C_{3}\gg C_{1}$, then for any $\mathcal{B}_{i}\subseteq\mathcal{Y}_{t}(x_{t})$,
\[
\mathcal{\mathbb{P}}(X_{0}\in\mathcal{B}_{i}\mymid X_{t}=x_{t})\le\exp\bigg(-\frac{\overline{\alpha}_{t}}{16(1-\overline{\alpha_{t}})}\Vert x_{i(x_{i})}^{\star}-x_{i}^{\star}\Vert_{2}^{2}\bigg)\mathbb{P}(X_{0}\in\mathcal{B}_{i}).
\]
With these two results in place, we have
\begin{align*}
\int_{\mathcal{X}_{t}(x_{t})}p_{X_{0}|X_{t}}(x_{0}\mymid x)\Vert x_{0}-\widehat{x}_{0}\Vert_{2}^{2}\mathrm{d}x_{0} & \le\sup_{x_{0}\in\mathcal{X}_{t}(x_{t})}\left(\Vert x_{i(x_{i})}^{\star}-x_{i}^{\star}\Vert_{2}+4\sqrt{\frac{C_{3}k(1-\overline{\alpha}_{t})\log T}{\overline{\alpha}_{t}}}\right)^{2}\\
 & \le25\frac{C_{3}k(1-\overline{\alpha}_{t})\log T}{\overline{\alpha}_{t}},
\end{align*}
and
\begin{align*}
 & \int_{\mathcal{Y}_{t}(x_{t})}p_{X_{0}|X_{t}}(x_{0}\mymid x)\Vert x_{0}-\widehat{x}_{0}\Vert^{2}_0\mathrm{d}x_{0}\leq\sum_{i\notin\mathcal{I}(x_{t};C_{3})}\mathbb{P}\left(X_{0}\in B_{i}\mymid X_{t}=x_{t}\right)\max_{x_{0}\in B_{i}}\Vert x_{0}-\widehat{x}_{0}\Vert_{2}^{2}\\
 & \qquad\qquad\le\sum_{i\notin\mathcal{I}(x_{t};C_{3})}\left\{ \exp\left(-\frac{\overline{\alpha}_{t}}{16(1-\overline{\alpha_{t}})}\Vert x_{i(x_{i})}^{\star}-x_{i}^{\star}\Vert_{2}^{2}\right)\mathbb{P}(X_{0}\in B_{i})\cdot25\Vert x_{i(x_{i})}^{\star}-x_{i}^{\star}\Vert_{2}^{2}\right\} \\
 & \qquad\qquad\le\frac{400(1-\overline{\alpha}_{t})}{e\overline{\alpha}_{t}},
\end{align*}
where the last relation follows from the inequality $e^{-ax}x\le1/(e\cdot a)$
for $x\ge0.$ Combining both bounds yields
\begin{align}
 & \frac{\overline{\alpha}_{t}}{1-\overline{\alpha}_{t}}\bigg(\int p_{X_{0}|X_{t}}(x_{0}\mymid x)\Vert x_{0}-\widehat{x}_{0}\Vert_{2}^{2}\mathrm{d}x_{0}\bigg)\nonumber \\
& \quad =  \frac{\overline{\alpha}_{t}}{1-\overline{\alpha}_{t}}\bigg(\int_{\mathcal{X}_{t}(x_{t})}p_{X_{0}|X_{t}}(x_{0}\mymid x)\Vert x_{0}-\widehat{x}_{0}\Vert_{2}^{2}\mathrm{d}x_{0}+\int_{\mathcal{Y}_{t}(x_{t})}p_{X_{0}|X_{t}}(x_{0}\mymid x)\Vert x_{0}-\widehat{x}_{0}\Vert^{2}_2\mathrm{d}x_{0}\bigg)\nonumber \\
& \quad \le  \frac{\overline{\alpha}_{t}}{1-\overline{\alpha}_{t}}\left(25\frac{C_{3}k(1-\overline{\alpha}_{t})\log T}{\overline{\alpha}_{t}}+\frac{400(1-\overline{\alpha}_{t})}{e\overline{\alpha}_{t}}\right)\le26C_{3}k\log T.\label{eq:14.3}
\end{align}
The last relation holds as long as $C_{3}$ is sufficient large. 

Finally, putting \eqref{eq:14}, \eqref{eq:14.0} and \eqref{eq:14.3}
together, we obtain that for any $x\in\mathcal{T}_{t}$, 
\begin{alignat}{1}
 & \frac{p_{\sqrt{\alpha_{t}}X_{t-1}}(\phi_{t}(x))}{p_{X_{t}}(x)}/\frac{p_{\phi_{t}(Y_{t})}(\phi_{t}(x))}{p_{Y_{t}}(x)}\label{eq:19}\\
 & \quad\ge\exp\Bigg(\Delta-5(1-\alpha_{t})\left[\varepsilon_{\mathsf{score},t}(x)+\varepsilon_{\mathsf{Jacobi},t}(x)\right]-4\left\Vert \sqrt{\frac{1-\overline{\alpha}_{t}}{\alpha_{t}-\overline{\alpha}_{t}}}\frac{\partial\phi_{t}^{\star}(x)}{\partial x}-I\right\Vert _{\mathrm{F}}^{2}\Bigg)\nonumber \\
 & \quad\ge\exp\Bigg(-1664c_{1}^{2}C_{3}\frac{k\log^{3}T}{T^{2}}-5(1-\alpha_{t})\left[\varepsilon_{\mathsf{score},t}(x)+\varepsilon_{\mathsf{Jacobi},t}(x)\right]-4\left\Vert \sqrt{\frac{1-\overline{\alpha}_{t}}{\alpha_{t}-\overline{\alpha}_{t}}}\frac{\partial\phi_{t}^{\star}(x)}{\partial x}-I\right\Vert _{\mathrm{F}}^{2}\Bigg),\nonumber 
\end{alignat}
given that $T$ is sufficient large, and
\[
\frac{1-\alpha_{t}}{1-\overline{\alpha}_{t}}\varepsilon_{\mathsf{score},t}(x)\le1,\quad(1-\alpha_{t})\varepsilon_{\mathsf{Jacobi},t}(x)\le\frac{1}{8},\quad\text{and}\quad\Big\Vert\sqrt{\frac{1-\overline{\alpha}_{t}}{\alpha_{t}-\overline{\alpha}_{t}}}\frac{\partial\phi_{t}^{\star}(x)}{\partial x}-I\Big\Vert\le\frac{1}{8}.
\]

\subsection{Proof of Lemma \ref{lem:i2} \protect\label{PF=0000202}}

For $2\le i\le T,$ we observe that
\begin{align}
 & \left(\mathcal{E}_{i}^{\mathrm{c}}\cap(\cap_{t>i}^{T}\mathcal{E}_{t})\right)\subseteq\underbrace{\left(\left(\cap_{t>i}^{T}\mathcal{E}_{t}\right)\cap\left\{ y_{T}:S_{i-1}(y_{T})>1,y_{T}\in\mathcal{T}_{T}\right\} \right)}_{:=J_{1,i}}\cup\underbrace{\left(\left(\cap_{t>i}^{T}\mathcal{E}_{t}\right)\cap\{y_{T}:y_{i}\notin\mathcal{T}_{i}\big\}\right)}_{:=J_{2,i}}.\label{eq:i2-1}
\end{align}
Note that we define $\cap_{t>i}^{T}\mathcal{E}_{t}=\mathbb{R}^{d}$
when $i=T$. We will analyze $J_{1,i}$ and $J_{2,i}$ separately. 

We can bound $\mathcal{\mathbb{P}}\left(Y_{T}\in J_{1,i}\right)$
as follows:
\begin{align}
\sum_{i=2}^{T}\mathbb{E}\left[\mathds{1}\left\{ Y_{T}\in J_{1,i}\right\} \right] & \overset{(\text{a})}{\le}\sum_{i=2}^{T}\mathbb{E}\left[S_{i-1}(Y_{T})\mathds{1}\left\{ Y_{T}\in J_{1,i}\right\} \right]=\sum_{i=2}^{T}\sum_{t=i}^{T}\mathbb{E}\left[\xi_{t}(Y_{t})\mathds{1}\left\{ Y_{T}\in J_{1,i}\right\} \right]\nonumber \\
 & =\sum_{i=2}^{T}\sum_{t=i}^{T}\int\xi_{t}(y_{t})p_{Y_{t}}(y_{t})\mathds{1}\left\{ y_{T}\in J_{1,i}\right\} \mathrm{d}y_{t}\nonumber \\
 & \overset{(\text{b})}{\le}e^{2}\sum_{i=2}^{T}\sum_{t=i}^{T}\int\xi_{t}(y_{t})p_{X_{t}}(y_{t})\mathds{1}\left\{ y_{T}\in J_{1,i}\right\} \mathrm{d}y_{t}\nonumber \\
 & \le e^{2}\sum_{t=2}^{T}\int\xi_{t}(y_{t})p_{X_{t}}(y_{t})\sum_{i=2}^{T}\mathds{1}\left\{ y_{T}\in J_{1,i}\right\} \mathrm{d}y_{t}\nonumber \\
 & \overset{(\text{c})}{\le}e^{2}\sum_{t=2}^{T}\mathbb{E}\left[\xi_{t}(X_{t})\right].\label{eq:i2-2}
\end{align}
Here, step (a) holds since $S_{i-1}(Y_{T})>1$ when $Y_{T}\in J_{1,i}$;
step (b) utilizes \eqref{eq:l2} and Lemma \eqref{lem-T}; while step
(c) follows from the observation that $J_{1,i}$ are disjoint.

For $J_{2,i}$, applying \eqref{eq:l2} again, we have for $2\le i\le T-1,$
\begin{align}
\mathbb{E}\left[\mathds{1}\left\{ Y_{T}\in J_{2,i}\right\} \right] & =\int_{y_{i}:y_{T}\in\cap_{t>i}^{T}\mathcal{E}_{t}}p_{Y_{i}}(y_{i})\mathds{1}\left\{ y_{i}\notin\mathcal{T}_{i}\right\} \mathrm{d}y_{i}\nonumber \\
 & \le e^{2}\int_{y_{i}}p_{X_{i}}(y_{i})\mathrm{\mathds{1}\left\{ y_{i}\notin\mathcal{T}_{i}\right\} d}y_{i}=e^{2}\mathbb{P}\left(X_{i}\notin\mathcal{T}_{i}\right).\label{eq:i2-3}
\end{align}

Combining \eqref{eq:i2-1}, \eqref{eq:i2-2} and \eqref{eq:i2-3}
leads to
\begin{align*}
I_{2} & \le\sum_{i=2}^{T}\mathbb{E}\bigg[\mathds{1}\big\{ Y_{T}\in J_{1,i}\big\}\bigg]+\sum_{i=2}^{T-1}\mathbb{E}\left[\mathds{1}\left\{ Y_{T}\in J_{2,i}\right\} \right]+\mathbb{E}\left[\mathds{1}\left\{ Y_{T}\notin\mathcal{T}_{T}\right\} \right]\\
 & \le e^{2}\sum_{t=2}^{T}\mathbb{E}\left[\xi_{t}(X_{t})\right]+e^{2}\sum_{t=2}^{T-1}\mathbb{P}\left(X_{t}\notin\mathcal{T}_{t}\right)+\mathcal{\mathbb{P}}\left(Y_{T}\notin\mathcal{T}_{T}\right)
\end{align*}
as claimed.

\section{Technical lemmas}

This section presents a few technical lemmas used throughout the proof.

\begin{lemma}\label{lem:T1}
When $T$ is sufficiently large, for $1\leq t\leq T$, 
\[
\alpha_{t}\geq1-\frac{c_{1}\log T}{T}\geq\frac{1}{2}.
\]
Moreover, for $2\leq t\leq T$, 
\begin{align*}
\frac{1-\alpha_{t}}{1-\overline{\alpha}_{t}} & \leq\frac{1-\alpha_{t}}{\alpha_{t}-\overline{\alpha}_{t}}\leq\frac{8c_{1}\log T}{T}.
\end{align*}
In addition, there exists some sufficient large constant $c_{2}\ge1000$ such that
\[
\overline{\alpha}_{T}\le\frac{1}{T^{c_{2}}}.
\]
\end{lemma}
\begin{proof} See \citep[Appendix A.2]{li2024sharp}.\end{proof}

\begin{lemma}\label{lemma:T3}Suppose that $T$ is sufficiently large
and that $c_{2}\gg2c_{R}+1$. Then we have 
\[
\mathsf{TV}\left(p_{X_{T}},p_{Y_{T}}\right)\le T^{-99}.
\]
\end{lemma}\begin{proof}See \citep[Lemma 3]{li2024sharp}.\end{proof} 

\begin{lemma}\label{lemma:logdet}For any matrix $A\in\mathbb{R}^{d\times d}$
satisfying $\Vert A\Vert\leq1/4$, we have
\[
\log\det(I+A)\geq\mathsf{Tr}(A)-2\Vert A\Vert_{\mathrm{F}}^{2}.
\]
\end{lemma}

\begin{proof} Notice that $\det(I+A)=\det(I+A^{\top})$, we have
\begin{align}
\log\det(I+A) & =\frac{1}{2}\log\det(I+A)+\frac{1}{2}\log\det(I+A^{\top})=\frac{1}{2}\log\det\big[(I+A)(I+A^{\top})\big]\nonumber \\
 & =\frac{1}{2}\log\det(I+A+A^{\top}+AA^{\top})\geq\frac{1}{2}\log\det(I+\Delta),\label{eq:logdet-1}
\end{align}
where we define $\Delta\coloneqq A+A^{\top}$. Here the last relation
follows from the a consequence of the Weyl's inequality:
\[
I+A+A^{\top}+AA^{\top}\succeq I+A+A^{\top}\succeq\frac{1}{2}I.
\]
Since $I+\Delta$ is a positive semi-definite matrix, we have
\begin{align}
\log\det(I+\Delta) & =\sum_{i=1}^{d}\log\lambda_{i}(I+\Delta)=\sum_{i=1}^{d}\log[1+\lambda_{i}(\Delta)]\overset{\text{(i)}}{\geq}\sum_{i=1}^{d}\lambda_{i}(\Delta)-\lambda_{i}^{2}(\Delta)\nonumber \\
 & \overset{\text{(ii)}}{=}\mathsf{Tr}(\Delta)-\Vert\Delta\Vert_{\mathrm{F}}^{2}\overset{\text{(iii)}}{\geq}2\mathsf{Tr}(A)-4\Vert A\Vert_{\mathrm{F}}^{2}\label{eq:logdet-2}
\end{align}
Here step (i) follows from the fact that $\log(1+x)\geq x-x^{2}$
for $x\geq0$; step (ii) holds since $\mathsf{Tr}(\Delta)=\sum_{i=1}^{d}\lambda_{i}(\Delta)$
and $\Vert\Delta\Vert_{\mathrm{F}}^{2}=\sum_{i=1}^{d}\lambda_{i}^{2}(\Delta)$;
step (iii) holds since 
\[
\Vert\Delta\Vert_{\mathrm{F}}^{2}=\Vert A\Vert_{\mathrm{F}}^{2}+\Vert A^{\top}\Vert_{\mathrm{F}}^{2}+2\langle A,A^{\top}\rangle\leq4\Vert A\Vert_{\mathrm{F}}^{2}.
\]
Taking (\ref{eq:logdet-1}) and (\ref{eq:logdet-2}) collectively
yields the desired result.
\end{proof}

\begin{lemma}\label{lemma:T4}Suppose that $T$ is sufficiently large.
Then there exists some universal constant $C_{6}>0$ such that
\[
\sum_{t=2}^{T}\mathbb{E}\left(\left\Vert \sqrt{\frac{1-\overline{\alpha}_{t}}{\alpha_{t}-\overline{\alpha}_{t}}}\frac{\partial\phi_{t}^{\star}(X_{t})}{\partial x}-I\right\Vert _{\mathrm{F}}^{2}\right)\le C_{6}\frac{k\log^{2}T}{T}.
\]
\end{lemma}
\begin{proof}Define a matrix function $\Sigma_{\overline{\alpha}_{t}}(\cdot)$
as 
\[
\Sigma_{\overline{\alpha}_{t}}(x):=\text{Cov}\left(Z\mymid\sqrt{\overline{\alpha}_{t}}X_{0}+\sqrt{1-\overline{\alpha}_{t}}Z=X_{t}\right),
\]
where $Z\sim\mathcal{N}(0,I_{d})$ is independent of $X_{0}$. We
have
\begin{align*}
\sum_{t=2}^{T}\mathbb{E}\left(\left\Vert \sqrt{\frac{1-\overline{\alpha}_{t}}{\alpha_{t}-\overline{\alpha}_{t}}}\frac{\partial\phi_{t}^{\star}(X_{t})}{\partial x}-I\right\Vert _{\mathrm{F}}^{2}\right) & \overset{\text{(a)}}{=}\sum_{t=2}^{T}\bigg(\sqrt{1+\frac{1-\alpha_{t}}{\alpha_{t}-\overline{\alpha}_{t}}}-1\bigg)^{2}\mathbb{E}\left[\left\Vert \Sigma_{\overline{\alpha}_{t}}(X_{t})\right\Vert _{\mathrm{F}}^{2}\right]\\
 & \overset{\text{(b)}}{\le}\bigg(\frac{1-\alpha_{t}}{2\left(\alpha_{t}-\overline{\alpha}_{t}\right)}\bigg)^{2}\sum_{t=2}^{T}\mathbb{E}\left[\left\Vert \Sigma_{\overline{\alpha}_{t}}(X_{t})\right\Vert _{\mathrm{F}}^{2}\right]\\
 & \overset{\text{(c)}}{\le}\frac{2c_{1}\log T}{T}\left(1+\frac{8c_{1}\log T}{T}\right)\sum_{t=2}^{T}\frac{1-\alpha_{t}}{1-\overline{\alpha}_{t}}\text{Tr}\left(\mathbb{E}\left[\left(\Sigma_{\overline{\alpha}_{t}}(X_{t})\right)^{2}\right]\right)\\
 & \overset{\text{(d)}}{\le}\frac{C_{6}k\log^{2}T}{T}
\end{align*}
for some sufficient large constant $C_{6}>0$. Here, step (a) follows
from \eqref{eq:8}; step (b) utilizes the inequality $\sqrt{1+x}\le1+\frac{1}{2}x$
for $x\ge-1$; step (c) uses Lemma \ref{lem:T1} and the fact that
$\Sigma_{\overline{\alpha}_{t}}(x)$ is symmetric; while step (d)
follows from \citep[Lemma 18]{li2024d} provided that $T$ is sufficient
large.

\end{proof} 
\bibliographystyle{apalike}
\bibliography{reference-diffusion}

\end{document}